\documentclass{article}
\pdfoutput=1

\usepackage[utf8]{inputenc}
\usepackage{amssymb, amsmath, amsthm}
\usepackage{algorithm, algpseudocode}
\usepackage{tikz}
\usepackage{pgfplotstable}
\usepackage{booktabs}
\usepackage{tabularx}
\usepackage{enumitem}
\usepackage{hyperref}
\newcommand{\textalg}[1]{\textnormal{\textproc{#1}}}

\usepackage{color}

\newtheorem{lemma}{Lemma}[section]
\newtheorem{definition}[lemma]{Definition}
\newtheorem{proposition}[lemma]{Proposition}
\newtheorem{theorem}[lemma]{Theorem}
\newtheorem{corollary}[lemma]{Corollary}
\newtheorem{remark}[lemma]{Remark}

\DeclareMathOperator*{\argmin}{argmin}
\DeclareMathOperator*{\argmax}{argmax}
\DeclareMathOperator{\supp}{supp}
\DeclareMathOperator*{\vecz}{vec}

\newcommand{\real}{\mathbb{R}}
\newcommand{\dualp}[1]{\left\langle #1 \right\rangle} 
\newcommand{\E}[1]{\mathbb{E}\left[ #1 \right]}
\newcommand{\pr}[1]{\operatorname{Pr}\left[ #1 \right]}
\newcommand{\norm}[1]{\left\| #1 \right\|}
\newcommand{\one}{\boldsymbol{1}}

\newcommand{\nolabel}[1]{}
\newcommand{\indexset}{\mathcal{I}}

\newcommand{\class}{\mathcal{C}}
\newcommand{\classE}[1][]{\mathcal{C}_{\text{easy#1}}}
\newcommand{\child}{\operatorname{child}}
\newcommand{\nc}{\gamma}
\newcommand{\nn}{N}

\newcommand{\outY}{\bar{Y}}
\newcommand{\outX}{\bar{X}}
\newcommand{\outZ}{\bar{Z}}
\newcommand{\variance}{\nu}
\newcommand{\psitwo}{C_\psi}
\newcommand{\ttrain}{\bar{t}}

\newcommand{\splitmat}{\mathcal{S}}
\newcommand{\suppx}{\mathcal{J}}
\newcommand{\suppz}{\mathcal{K}}
\newcommand{\suppxz}{\mathcal{JK}}

\newcommand{\vecA}{\mathcal{A}}
\newcommand{\vecR}{\mathcal{R}}
\newcommand{\vecb}{\mathcal{B}}

\newcommand{\palt}{p'}

\begin{document}

\title{Learning Trees of $\ell_0$-Minimization Problems}
\author{G. Welper\footnote{Department of Mathematics, University of Central Florida, Orlando, FL 32816, USA, email \href{mailto:gerrit.welper@ucf.edu}{\texttt{gerrit.welper@ucf.edu}}.}}
\date{}
\maketitle

\begin{abstract}

  The problem of computing minimally sparse solutions of under-determined linear systems is $NP$ hard in general. Subsets with extra properties, may allow efficient algorithms, most notably problems with the restricted isometry property (RIP) can be solved by convex $\ell_1$-minimization. While these classes have been very successful, they leave out many practical applications.

  In this paper, we consider adaptable classes that are tractable after training on a curriculum of increasingly difficult samples. The setup is intended as a candidate model for a human mathematician, who may not be able to tackle an arbitrary proof right away, but may be successful in relatively flexible subclasses, or areas of expertise, after training on a suitable curriculum.

\end{abstract}



\section{Introduction}

We consider efficiently solvable subclasses of $NP$ hard problems, variations of 3SAT at the end of the paper and sparse solutions of linear systems in its main part: For matrix $A \in \real^{m \times n}$ and right hand side $b \in \real^m$, we wish to find the sparsest solution of
\begin{equation} \label{eq:cs0}
  \begin{aligned}
    & \min_{x \in \real^n} \|x\|_0 & & \text{subject to} & A x & = b,
  \end{aligned}
\end{equation}
where $\|x\|_0$ denotes the number of non-zero entries of $x$. In full generality, this problem is $NP$-hard
\cite{
  Natarajan1995,
  GeJiangYe2011%
}
but as many hard problems it contains tractable subclasses. Some of these are uninteresting, at least form the perspective of sparsity, e.g. problems with zero kernel $\ker(A) = 0$ and unique solution, which renders the $\ell_0$-minimization trivial. Other tractable subclasses have been extensively studied in the literature, most notable problems that satisfy the \emph{$(s,\epsilon)$-Restricted Isometry property (RIP)}
\begin{equation*}
  \begin{aligned}
    (1-\epsilon) \|x\| & \le \|Ax\| \le (1+\epsilon)\|x\| &
    & \text{for all $s$-sparse }x \in \real^n,
  \end{aligned}
\end{equation*}
with strict requirements $\epsilon < 4/\sqrt{41} \approx 0.6246$ on the RIP parameters and more generally the \emph{null space property (NSP) of order $s$}
\begin{equation*}
  \begin{aligned}
    \|v_S\|_1 & < \|v_{\bar{S}}\|_1 &
    & \text{for all }0 \ne v \in \operatorname{ker} A\text{ and } |S| \le s,
  \end{aligned}
\end{equation*}
where $v_S$ is the restriction of $v$ to an index set $S$ and $\bar{S}$ its complement. In both cases, the sparsest solution of \eqref{eq:cs0} is found by the relaxation of the sparsity $\|\cdot\|_0$ to the convex $\|\cdot\|_1$-norm
\begin{equation*}
  \begin{aligned}
    & \min_{x \in \real^n} \|x\|_1 & & \text{subject to} & A x & = b,
  \end{aligned}
\end{equation*}
see \cite{CandesRombergTao2006a,Donoho2006,CandesRombergTao2006,FoucartRauhut2013} for details.

All of these tractable subclasses are completely rigid: A problem is either contained in the class or we are out of luck. Alternatively, there are subclasses based on prior knowledge. Trivially, if we know that the solution $x = Xz$ is in the column span of a matrix $X \in \real^{n \times p}$, we can simplify the search space
\begin{equation*}
  \begin{aligned}
    & \min_{z \in \real^p} \|Xz\|_0 & & \text{subject to} & A X z & = b,
  \end{aligned}
\end{equation*}
or even simpler
\begin{equation} \label{eq:cs0-X}
  \begin{aligned}
    & \min_{z \in \real^p} \|z\|_0 & & \text{subject to} & A X z & = b,
  \end{aligned}
\end{equation}
if $X$ has sparse columns. Again, we can find tractable subclasses, where $AX$ is injective or where $AX$ satisfies the RIP
\cite{
  KasiviswanathanRudelson2019,
  Welper2020,
  Welper2021%
}.
With relatively simple rank constraints on $A$, these classes can contain every possible solution $x$, but they are useless without explicit knowledge of $X$. A purely computational approach to uncover $X$ is not promising because it would provide us with efficient algorithms for generic $NP$ hard problems. Instead of addressing a difficult $\ell_0$-minimization instance heads on, we therefore consider a sequence of $\ell_0$-minimization instances organized into a curriculum of separate learning episodes, each one consisting of samples from a different tractable subclasses of increasing difficulty. 
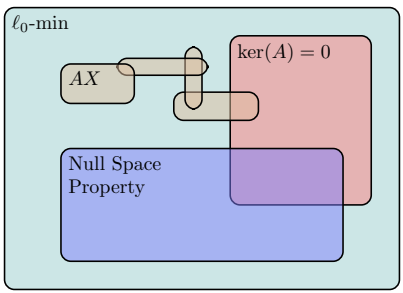
\begin{center}
  \scalebox{0.75}{\begin{tikzpicture}
    \draw[black,fill=teal!20,rounded corners=5,thick] (0,0) rectangle (-7,5)
    node[below right] {$\ell_0$-min};

    \draw[black,fill=red!50,fill opacity=0.5,rounded corners=5,thick] (-0.5,1.5) rectangle (-3,4.5)
    node[below right, opacity=1.0] {$\ker(A) = 0$};

    \draw[black,fill=blue!50,fill opacity=0.5,rounded corners=5,thick] (-1,0.5) rectangle (-6,2.5)
    node[below right, opacity=1.0, text width=2cm] {Null Space Property};

    \draw[black,fill=brown!50,fill opacity=0.5,rounded corners=5,thick] (-2.5,3) rectangle (-4,3.5);
    \draw[black,fill=brown!50,fill opacity=0.5,rounded corners=5,thick] (-3.5,3.2) rectangle (-3.8,4.3);
    \draw[black,fill=brown!50,fill opacity=0.5,rounded corners=5,thick] (-3.4,3.8) rectangle (-5,4.1);
    \draw[black,fill=brown!50,fill opacity=0.5,rounded corners=5,thick] (-4.7,3.3) rectangle (-6,4)
    node[below right, opacity=1.0] {$AX$};

  \end{tikzpicture}}
\end{center}

In order to follow a chain of learning episodes, we use a mechanism to learn a full class from simple samples, introduced in \cite{Welper2021} and summarized in Section \ref{sec:easy-and-hard}. Simple problems are ones that can be efficiently solved by a student who has mastered prerequisite problem classes organized in a curriculum or tree, Section \ref{sec:class-tree}. In Section \ref{sec:model-tree}, we construct an example for a tree that enables the student to find an arbitrary $\ell_0$ minimizer $x$ together with further random solutions, added to model more realistic problem classes of non-trivial size. Finally, in Section \ref{sec:3SAT}, we apply the learning method to a signed variant of $NP$ complete 1-in-3-SAT problems.

\paragraph{Human Learning}

The prior knowledge informed subclasses, together with and iterative learning curriculum, are intended as a hypothetical model for human problem solving, or more concretely theorem proving. 

If $N \ne NP$, and humans brains have no fundamental superiority to computers, humans cannot effectively solve arbitrary instances of computationally hard problems. Yet, we routinely prove theorems and have build up a rich trove of results. But we only do so in our respective areas of expertise. Hence, one may argue that within these areas, and equipped with prior knowledge and experience, theorem proving is tractable. If so, can we program corresponding solvers into a computer? The history of artificial intelligence provides some caution. Hand coded rules in expert systems and natural language processing have proven difficult due to their immense complexity, while learned approaches are currently superior. Likewise, instead of hand crafting tractable subclasses, it seems more promising to learn them.

As a mathematical model for tractable subclasses, we consider sparse solutions of linear systems. These are $NP$-hard and in \eqref{eq:cs0-X}, we have already identified some adaptable and tractable subclasses. The solution vector $x$ is a model for a proof, as both are hard to compute. The linear combination $x = Xz$, together with the non-linear minimal sparsity, composes a candidate solution $x$ form elementary pieces in the columns of $X$, similar to assembling a proof form known tricks, techniques, lemmas and theorems.

Of course, this solution strategy is of no use if we do not know $X$. Likewise, humans need to acquire their expertise, either through training or research. An important component of both, is the solution of many related and often simplified problems. For a student, these are split into episodes, ordered by prerequisites into a curriculum tree. Likewise, for our mathematical model, we learn a tree of subclasses $X_i$ from simple samples, i.e. pairs $(A_k, b_k)$ in the respective classes.

As we will see (Remark \ref{remark:meta-heuristic}), the combined knowledge of all descendant nodes $[X_1, X_2, \dots]$ is not sufficient to solve all problems in the root node $X_0$ because in an expansion $x = X_0 z_0 = \sum_i X_i z_i$, the $z_i$ combined generally have less sparsity than $z_0$ and are thus more difficult to find. Therefore, at each tree node we compress our knowledge into matrices with less columns and more sparse $z$. This step is similar to summarizing reoccurring proof steps into a lemma and the using it as a black box in subsequent classes.

\paragraph{Greedy Search and Heuristics}

Similar to $\ell_1$ minimization, greedy algorithms like \emph{orthogonal matching pursuit}
\begin{align*}
  j^{n+1} & = \argmax_j \left| A_{\cdot j}^T (A x^n - b) \right| \\
  S^{n+1} & = S^n \cup \{j^{n+1}\} \\
  x^{n+1} & = \argmin_{\supp(x) \subset S^{n+1}} \|Ax - b\|_2^2,
\end{align*}
also find global $\ell_0$-minimizers under RIP assumptions \cite{FoucartRauhut2013}. Instead of systematically searching through an exponentially large set of candidate supports $S$, the first line provides a criterion to greedily select the next support index, based on the correlation of a column $A_{\cdot j}$ with the residual $Ax^n - b$. Applied to the modified problem \eqref{eq:cs0-X} with prior knowledge $X$, the method changes to
\begin{align*}
  j^{n+1} & = \argmax_j \left| X_{\cdot j}^T A^T (A X z^n - b) \right| \\
  S^{n+1} & = S^n \cup \{j^{n+1}\} \\
  z^{n+1} & = \argmin_{\supp(z) \subset S^{n+1}} \|A X z - b\|_2^2.
\end{align*}
In the first row, the learned knowledge $X$ modifies the index selection and thus provides a learned greedy criterion or heuristic. The learning of $X$, however, implicitly depends on a meta-heuristic as explained in Remark \ref{remark:meta-heuristic} below. From this perspective, the proposed methods are related to greedy and heuristic search methods in AI
\cite{
  RussellNorvigDavis2010,
  SuttonBarto2018,
  Holden2021%
}.

\paragraph{\texorpdfstring{$\ell_0$}{l0}-Minimization without RIP}

This paper is mainly concerned with minimally sparse solutions of systems with non-NSP or non-RIP matrices $A$. A common approach in the literature for these systems is $\ell_p$-minimization with $p<1$, which resembles the $\ell_0$-norm more closely than the convex $\ell_1$ norm. While sparse recovery can be guaranteed for weaker variants of the RIP
\cite{
  CandesWakinBoyd2008,
  ChartrandStaneva2008,
  FoucartLai2009,
  Sun2012,
  ShenLi2012%
},
these problems are again $NP$ hard \cite{GeJiangYe2011}. Nonetheless, iterative solvers for $\ell_p$-minimization or non-RIP $A$ often show good results
\cite{
  CandesWakinBoyd2008,
  ChartrandWotaoYin2008,
  FoucartLai2009,
  DaubechiesDeVoreFornasierEtAl2010,
  LaiXuYin2013,
  WoodworthChartrand2016%
}.

\paragraph{\texorpdfstring{$\ell_0$}{l0}-Minimization with Learning}

Similar to our approach, many papers study prior information for under-determined linear systems $Ax = b$. Similar to this paper, $\ell_1$ synthesis \cite{MarzBoyerKahnWeiss2022} considers solutions of the form $x = Xz$, in case $x$ is not sparse in the standard basis and for random $A$. The papers
\cite{
  BoraJalalPriceEtAl2017,
  HandVoroninski2018,
  HuangHandHeckelEtAl2018,
  DharGroverErmon2018,
  WuRoscaLillicrap2019%
}
assume that the solution $x$ is in the range of a neural network $x = G(z;w)$, with weights pre-trained on relevant data, and then minimize $\min_z \|AG(z;w) - b\|_2$. Alternatively, the deep image prior \cite{UlyanovVedaldiLempitsky2020} and compressed sensing applications
\cite{
  VeenJalalSoltanolkotabiEtAl2020,
  JagatapHegde2019,
  HeckelSoltanolkotabi2020%
}
use the architecture of an untrained network as prior and minimize the weights $\min_w \|AG(z;w) - b\|_2$ for some latent input $z$. These papers assume i.i.d. Gaussian $A$ or the Restricted Eigenvalue Condition (REC) and use the prior to select a suitable candidate among all non-unique solutions. In contrast, in the present paper, we aim for the sparsest solution and use the prior to address the hardness of the problem.

The paper \cite{WuDimakisSanghaviYuHoltmannRiceStorcheusRostamizadehKumar2019} considers an auto-encoder mechanism to find measurement matrices $A$, not only $X$, as in our case. Several other papers that combine compressed sensing with machine learning approximate the right hand side to solution map $b \to x$ by neural networks
\cite{
  MardaniSunDonohoEtAl2018,
  ShiJiangZhangEtAl2017%
}.

\paragraph{Transfer Learning}

The progression through a tree splits the learning problem into separate episodes on different but related data sets. This is reminiscent of empirical studies on transfer- \cite{DonahueJiaVinyalsEtAl2014,YosinskiCluneBengioEtAl2014} and meta-learning \cite{HospedalesAntoniouMicaelliEtAl2020} in neural networks. 

\subsection{Notations}

We use $c$ and $C$ for generic constants, independent on dimension, variance or $\psi_2$ norms that can change in each formula. We write $a \lesssim b$, $a \gtrsim b$ and $a \sim b$ for $a \le c b$, $a \ge c b$ and $c a \le b \le Cb$, respectively. We denote index sets by $[n] = \{1, \dots, n\}$ and restrictions of vectors, matrix rows and matrix columns to $J \subset [n]$ by $v_J$, $M_{J \cdot}$ and $M_{\cdot J}$, respectively.

\section{Easy and Hard Problems}
\label{sec:easy-and-hard}

\subsection{\texorpdfstring{$\ell_0$}{l0}-Minimization with Prior Knowledge}

For given matrix $A \in \real^{m \times n}$ and vector $b \in \real^m$, we consider the $\ell_0$-minimization problem
\begin{equation*}
  \begin{aligned}
    & \min_{x \in \real^n} \|x\|_0, &
    & \text{s.t.}
    & Ax & = b
  \end{aligned}
\end{equation*}
from the introduction. We have seen that this problem is $NP$-hard in general, but tractable for suitable subclasses. While the RIP and NSP conditions are rigid classes, fully determined by the matrix $A$, we now consider some more flexible ones, based on the prior knowledge that the solution is in some subset
\[
  \class_{< t} := \{x \in \real^n : \, x = X z, \, z\text{ is }t\text{-sparse} \},
\]
parametrized by some matrix $X \in \real^{n \times p}$ and with only mild assumptions on $A$ to be determined below. We may regard $X$'s columns as solution components and hence assume that they are $s$-sparse, as well, for some $s>0$, so that the solutions $x=Xz$ in class are $st$ sparse. Although the condition seems linear on first sight, the sparsity requirement of $z$ can lead to non-linear behaviour as explored in detail in \cite{Welper2021}. As usual, we relax the $\ell_0$ to $\ell_1$ norm and solve the convex optimization problem
\begin{equation} \label{eq:l1-with-prior}
  \begin{aligned}
    & \min_{x \in \real^n} \|z\|_1, &
    & \text{s.t.}
    & AXz & = b.
  \end{aligned}
\end{equation}
Of course any solver requires explicit knowledge of $X$, which we discuss in detail below. For now, let us assume $X$ is known. Two extreme cases are noteworthy. First, without prior knowledge $X=I$, we retain standard $\ell_1$-minimization
\begin{equation*}
  \begin{aligned}
    & \min_{x \in \real^n} \|x\|_1, &
    & \text{s.t.}
    & Ax & = b,
  \end{aligned}
\end{equation*}
which provides correct solutions for the $\ell_0$-minimization problem if $A$ satisfies the null-space property (NSP) or the restricted isometry property (RIP), typically for sufficiently random $A$.

Second, if instead of the matrix $A$, the prior knowledge $X$ is sufficiently random, we can reduce the null-space property of $A$ to a much weaker stable rank condition on $A$. In that case, the product $AX$ satisfies a RIP with high probability \cite{KasiviswanathanRudelson2019} and hence we can recover a unique sparse $z$. Since $X$ is also sparse, this leads to a sparse solution $x=Xz$ of the linear system $Ax=b$. However, we need some more structure to ensure that $x$ is indeed the $\ell_0$ optimizer. One possibility is to assume that all sparse solutions of $Ax=b$ are unique, which is similar to the RIP without any restrictive limitations on the constants and therefore much weaker. Alternatively, in Section \ref{sec:3SAT}, we consider reductions from $NP$-complete problems to $\ell_0$-minimization. These come with efficient verification of solutions, which we use to ensure that $x=Xz$ is the $\ell_0$-minimizer.

\subsection{Learning Prior Knowledge}

We have seen that subclasses $\class_{<t}$ of $\ell_0$-minimization problems may be tractable, given suitable prior knowledge encoded in the matrix $X$. Hence, we need a plausible model to acquire this knowledge. To this end, we consider a teacher - student scenario, with a teacher that provides sample problems and a student that infers knowledge $X$ from the samples.

The training samples must be chosen with care. Indeed, to be plausible for a variety of machine learning scenarios, we assume that the student receives samples $(A, b_i)$, but not the corresponding solutions $x_i$. How, then, can the student find the solutions $x_i$ without knowing $X$ yet, which is the very matrix she is supposed to learn?

To resolve this problem, the student trains on a subclass $\classE \subset \class_{<t}$ of easy problems that the she can (mostly) solve effectively without prior knowledge, denoted by
\begin{center}
  \begin{tabularx}{\linewidth}{lX}
    \textalg{Solve}$(A, b)$: & Compute the $\ell_0$-minimizer of $Ax = b$, $x \in \classE$.
  \end{tabularx}
\end{center}
For comparison, the presence of easy problems may also play a role in gradient descent training of neural networks \cite{Allen-ZhuLi2020}. At this point, we do not consider the implementation of the solver. It can be plain $\ell_1$-minimization, or $\ell_1$-minimization with prior knowledge from a previous learning episodes as discussed in Section \ref{sec:class-tree} below.

The student combines the easy solutions from $\classE$ into a matrix $Y$ (as columns). Since $\classE$ is contained in $\class_{<t}$, they must be of the form $Y = XZ$ for some $t$-sparse matrix $Z$. Given that $Y$ contains sufficiently many independent samples form the class $\class_{<t}$, sparse factorization algorithms
\cite{
  AharonEladBruckstein2006,
  GribonvalSchnass2010,
  SpielmanWangWright2012,
  AgarwalAnandkumarJainEtAl2014,
  AroraGeMoitra2014,
  AroraBhaskaraGeEtAl2014,
  NeyshaburPanigrahy2014,
  AroraGeMaEtAl2015,
  BarakKelnerSteurer2015,
  Schnass2015,
  SunQuWright2017,
  SunQuWright2017a,
  RenckerBachWangPlumbley2019,
  ZhaiYangLiaoEtAl2020%
}
can recover the matrices $X$ and $Z$ up to scaling $\Gamma$ and permutation $P$.
\begin{center}
  \begin{tabularx}{\linewidth}{lX}
    \textalg{SparseFactor}$(Y)$: & Factorize $Y$ into $\bar{X} = X P \Gamma$ and $\bar{Z} = \Gamma^{-1} P^{-1} Z$ for some permutation $P$ and diagonal scaling $\Gamma$. \\
    \textalg{Scale}: & Scale the columns of $\outX$ so that $A\outX$ satisfies the $RIP$.
  \end{tabularx}
\end{center}
The permutation is irrelevant, but we need proper scaling for $\ell_1$ minimizers to work, computed by \textalg{Scaling}, which is a simple normalization in \cite{Welper2021} and an application dependent function in Section \ref{sec:3SAT}. We combine the discussion into the following learning algorithm.
\begin{algorithm}
  \begin{algorithmic}
    \Function{Train}{$A$, $b_1, \dots, b_q$}
    \State For all $l \in [q]$, compute $y_l = \textalg{Solve}(A, b_l)$.
    \State Combine all $y_l$ into the columns of a matrix $\outY$.
    \State Compute $\outX, \, \outZ = \textalg{SparseFactor}(\outY)$
    \State \Return \textalg{Scale}($\outX$).
    \EndFunction
  \end{algorithmic}
  \caption{Training of easy problems $\classE$.}
  \label{alg:train}
\end{algorithm}

\begin{remark}
  In general $\outY$ and $\outX$ have the same column span and thus every $x \in \class_{<t}$ is given by
  \[
    x = \outX z = \outY u.
  \]
  Why don't we skip the sparse factorization? While $z$ is $t$-sparse by construction, $u = Y^+ x$ is generally not. Hence, even if $Y$ is sufficiently random for $AY$ to satisfy an RIP, it is not clear that it allows us to recover $u$ by the modified $\ell_1$-minimization \eqref{eq:l1-with-prior}.
\end{remark}

\subsection{Results}
\label{sec:easy-hard-results}

This section contains rigorous results for the algorithms of the last sections. First, we need a suitable model of random matrices.
\begin{definition}
  \label{def:bernoulli-subgaussian}

  A matrix $M \in \real^{n \times p}$ is \emph{$s/n$-Bernoulli-Subgaussian} if $M_{jk} = \Omega_{jk} R_{jk}$, where $\Omega$ is an i.i.d. Bernoulli matrix and $R$ is an i.i.d. Subgaussian matrix with
  \begin{equation}
  \begin{aligned}
    \E{\Omega_{jk}} & = \frac{s}{n}, &
    \E{R_{jk}} & = 0, &
    \E{R_{jk}^2} & = \variance^2, &
    \|R_{jk}\|_{\psi_2} & \le \variance \psitwo.
  \end{aligned}
  \label{eq:bernoulli-subgaussian}
  \end{equation}
  We call $M$ \emph{restricted $s/n$ Bernoulli-Subgaussian} if in addition
  \begin{equation}
  \begin{aligned}
    \pr{R_{jk} = 0} & = 0, &
    \E{|R_{jk}|} & \in \left[ \frac{1}{10}, 1\right], &
    \E{R_{jk}^2} & \le 1, &
    \pr{|R_{jk}| > \tau} & \le 2 e^{ \frac{-\tau^2}{2}}.
  \end{aligned}
  \label{eq:restricted-bernoulli-subgaussian}
  \end{equation}
\end{definition}
Recall that the $\psi_2$ norm is defined by $\|X\|_{\psi_2} := \sup_{p \ge 1} p^{-1/2} \E{|X|^p}^{1/p}$.

The first result states some necessary conditions for the training algorithm to recover $X$ up to perturbation and scaling. To this end, the teacher generates the training samples randomly, according to the following model.
\begin{enumerate}[label=(A\arabic*)]
  \item \label{assumption:sparse-factorization-alternative} The easy class $\classE$ is defined by pairs $(A, b_l)$ for $b_l = AXz_l$ with columns $z_l$ of $\ttrain/2p$ restricted Bernoulli-Subgaussian matrix $Z \in \real^{p \times q}$ with
  \begin{align} \label{eq:assumption:sparse-factorization-alternative}
    c \log q & \le \ttrain \le t , &
    q & > c p^2 \log^2 p, &
    \frac{2}{p} & \le \frac{\ttrain}{p} \le \frac{c}{\sqrt{p}}.
  \end{align}
\end{enumerate}
The vectors $z_l$ have expected sparsity $\ttrain$ and thus the corresponding solutions $X z_l$ have expected sparsity $s \ttrain$. In order for them be easier than the full class $\class_{<t}$, we generally choose $\ttrain < t$. Next, we require the student to be accurate on easy problems, with a safety margin $\sqrt{2}$ on sparsity
\begin{enumerate}[resume*]
  \item \label{assumption:enough-simple} For all $\sqrt{2}t$ sparse columns $z_l$ of $Z$, we have $\textalg{Solve}(A, AXz_l) = Xz_l$.
\end{enumerate}
It is crucial that $X z_l$ can be recovered from the data$A X z_l$, not necessarily that $X z_l$ is the globally sparsest solution of $Ax = AXz_l$, although that is usually the intention. Finally, we need the following technical assumption.
\begin{enumerate}[resume*]
  \item \label{assumption:rank} $X$ has full column rank.
\end{enumerate}
Although this implies that $X$ has more rows than columns, that is generally not true for $AX$ used in the sparse recovery \eqref{eq:l1-with-prior}. The assumption results from the sparse factorization \cite{SpielmanWangWright2012}, where $X$ represents a basis. Newer results
\cite{
  AgarwalAnandkumarJainEtAl2014,
  AroraGeMoitra2014,
  AroraBhaskaraGeEtAl2014,
  AroraGeMaEtAl2015,
  BarakKelnerSteurer2015%
}
consider over-complete bases with less rows than columns and coherence conditions and may eventually allow a weaker assumption. Anyways, with the given setup, we obtain the following training result.

\begin{theorem}[{\cite[Theorem 4.2]{Welper2021}}]
  \label{th:train}
  Assume that
  \ref{assumption:sparse-factorization-alternative}, \ref{assumption:enough-simple} and \ref{assumption:rank} hold. Then there are constants $c > 0$ and $C \ge 0$ independent of the probability model, dimensions and sparsity, and a tractable implementation of \textalg{SparseFactor} so that with probability at least
  \[
    1 - C p^{-c}
  \]
  the output $\outX$ of Algorithm \ref{alg:train} is a scaled permutation permutation $\outX = X P \Gamma$ of the matrix $X$ that defines the class $\class_{<t}$.
\end{theorem}

The result follows from Theorem $4.2$ in \cite{Welper2021} with some minor modifications described in Appendix \ref{appendix:easy-hard}.
After we have learned $X$, we need to ensure that we can solve all problems in class $\class_{<t}$ by \eqref{eq:l1-with-prior}, not only the easy ones. We show this for random $X$:
\begin{enumerate}[resume*]
  \item \label{assumption:rip} The matrix $X \in \real^{n \times p}$ is $s/n\sqrt{2}$ Bernoulli-Subgaussian with
  \begin{equation}
    \frac{\|A\|_F^2}{\|A\|^2} \ge C \psitwo^4 \frac{n t}{s \epsilon^2} \log \left( \frac{3p}{\epsilon t} \right)
    \label{eq:assumption:rip}
  \end{equation}
  and $\psi_2$-norm bound $\psitwo$ in the Bernoulli-Subgaussian model \eqref{eq:bernoulli-subgaussian}.
\end{enumerate}
The left hand side $\|A\|_F^2 / \|A\|^2$ is the stable rank of $A$. With the scaling
\begin{equation} \label{eq:scaling-factor}
  \textalg{Scale}(\outX) = \frac{\sqrt{n}}{\|A\|_F},
\end{equation}
we obtain the following result, with some minor modifications from the reference described in Appendix \ref{appendix:easy-hard}.

\begin{theorem}[{\cite[Theorem 4.2]{Welper2021}}]
  \label{th:train-rip}
  Assume we
  choose \eqref{eq:scaling-factor} for \textalg{Scale} and that \ref{assumption:sparse-factorization-alternative} and \ref{assumption:rip} hold. Then there are constants $c > 0$ and $C \ge 0$ independent of the probability model, dimensions and sparsity, and a tractable implementation of \textalg{SparseFactor} so that with probability at least
  \[
    1 - C p^{-c}
  \]
  the matrix $X$ has full column rank, $s$-sparse columns and $AX$ and satisfies the RIP
  \begin{equation}
    (1-\epsilon) \|v\|_2
    \le \|A \outX v \|_2
    \le (1+\epsilon) \|v\|_2
    \label{eq:th:train:RIP}
  \end{equation}
  for all $t$-sparse vectors $v \in \real^p$.
\end{theorem}

In conclusion, if we train on easy samples in $\classE$, we can recover $X$ and thus with the modified $\ell_1$-minimization \eqref{eq:l1-with-prior} solve all problems in class $\class_{<t}$, even the ones which we could not solve before training.

\subsection{Implementation of the Student Solver?}
\label{sec:easy-hard}

How can the student \textalg{Solve} easy problems $\classE$? If we implement \textalg{Solve} by plain $\ell_1$-minimization, $A$ must satisfy the $s\ttrain$ NSP. This poses strong assumptions on $A$ and if it satisfies the slightly stronger $st$ NSP, all problems in $\class_{<t}$ can be solved by $\ell_1$-minimization, rendering the training of $X$ obsolete. We resolve the issue in the next section by a hierarchy of problem classes, which allow us to use prior knowledge from lower level classes to implement \textalg{Solve}.

\section{Iterative Learning}
\label{sec:class-tree}

\subsection{Overview}

We have seen that we can learn to solve all problems in a class $\class_{<t}$, if we are provided with samples from an easier subclass $\classE$. The easy class must be sufficiently rich and at the same time its sample problems must be solvable without prior training. This results in a delicate set of assumptions. The situation becomes much more favorable if we do not try to learn $\class_{<t}$ at once, but instead iteratively proceed from easy to harder and harder problems. To this end, we order multiple problem classes into a curriculum, similar to a human student who progresses from easy to hard classes ordered by a set of prerequisites. Likewise, we consider a collection of problem classes $\class_i$, indexed by some index set $i \in \indexset$ and organized in a tree, e.g.
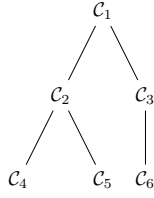
\begin{center}
  \scalebox{0.75}{\begin{tikzpicture}
    \node {$\class_1$}
      child {node {$\class_2$}
        child {node {$\class_4$}}
        child {node {$\class_5$}}
      }
      child {node {$\class_3$}
        child {node {$\class_6$}}
      }
    ;
  \end{tikzpicture}}
\end{center}
with root node $C_0$ and where each class $\class_i$ has children $\class_j$, $j \in \child(i)$. The student starts learning the leaves and may proceed to a class $\class_i$ only if all prerequisite or child classes have been successfully learned. As before each class is given by a matrix $X_i$ with $s_i$ sparse columns and sparsity $t$
\[
  \class_i := \{x \in \real^n : \, x = X_i z, \, z\text{ is }t\text{-sparse} \}.
\]
The difficulty of each class roughly corresponds to the sparsity, with the easiest at the leaves and then less and less sparsity towards the root of the tree. In order to learn each class $\class_i$, the corresponding easy problems are constructed as in the last section
\[
  \classE[,i] := \{x \in \real^n : \, x = X_i z, \, z\text{ is }\ttrain\text{-sparse} \},
\]
which are identical to $\class_i$ but with sparser vectors $z$.

It is crucial that instances in the easy class $\classE[,i]$ can be solved effectively by some solver \textalg{Solve}. While this leads to impractical assumptions in Theorem \ref{th:train}, this time problems are avoided by leveraging the outcome matrices $X_j$, $j \in \child(i)$ of the prerequisite problem classes. Indeed, we choose the curriculum so that all easy problems are contained in the combination of all children
\begin{equation} \label{eq:X-child}
  \begin{aligned}
    X_i & = \sum_{j \in \child(i)} X_j W_j
    =: X_{\child(i)} W_{\child(i)}
  \end{aligned}
\end{equation}
and extend the $t$-NSP of $AX_i$, $i \in \indexset$ to all its siblings, so that in particular $AX_{\child(i)}$ is $t$-NSP. The columns of $W_{\child(i)}$ have carefully calibrated sparsity of $t / \ttrain > 1$ or less so that
\begin{align} \label{eq:sparsity-constraints}
  t / \ttrain & \ge 1, &
  s_i & \ttrain \le s_j t
\end{align}
and thus
\[
  \begin{aligned}
    & x \in \text{Child problems}\leadsto & x & = X_{\child(i)} z_{\child(i)}, & \|z_{\child(i)}\|_0 & \le t , & \|x\|_0 \le s_j t, \\
    & x \in \classE[,i]\leadsto & x & = X_i z_i, & \|z_i\|_0 & \le \ttrain, & \|x\|_0 \le s_i \ttrain, \\
    & x \in \class_i\leadsto & x & = X_i z_i, & \|z_i\|_0 & \le t, & \|x\|_0 \le s_i t.
  \end{aligned}
\]
Initially the students satisfies the prerequisites and hence knows $X_{\child(i)}$. Thus, she can find all $s_j t$ sparse solutions $x = X_{\child(i)} z_{\child(i)}$, which include the $s_i \ttrain \le s_j t$ sparse easy solutions $X_i z_i$ and therefore provide an implementation of \textalg{Solve}. Solutions in the full class $\class_i$ are generally only $s_j t^2 / \ttrain > s_j t$ sparse linear combinations of $X_{\child(i)}$ and therefore not yet accessible by the student. But with the implementation of \textalg{Solve}, she can apply Algorithm \ref{alg:train} and learn $X_i$, and therefore the entire class $\class_i$. Combining $X_i$ with its siblings, the student can repeat the procedure and inductively move up in the curriculum tree. The split \eqref{eq:X-child}, roughly models a set of university courses, where higher level courses recombine concepts from multiple prerequisite courses.

It remains to learn the leaves, for which we cannot rely on any prior knowledge. Ideally, they are of unit sparsity $\mathcal{O}(1)$, which can be solved by brute force in sub-exponential time. For some applications this may be costly, while for others, like SAT reductions to compressed sensing and related problems discussed in Section \ref{sec:3SAT}, this is routinely done for moderately sized problems \cite{Holden2021}.

\begin{remark}
  All problems $x$ in class $\class_i$ are $t^2/\ttrain$-sparse linear combinations of $X_{\child(i)}$. Hence, if $AX_{\child(i)}$ satisfies the $t^2/\ttrain$ instead of only a $t$-NSP, the student can solve all problems in $\class_i$, without training Algorithm \ref{alg:train}. Practically, she can jump a class, but it is increasingly difficult to jump all classes, which would render the entire learning procedure void.
\end{remark}

\begin{remark}
  The easy/hard split is achieved by some matrix satisfying a $\ttrain$ but not a $t$ RIP. In Section \ref{sec:easy-and-hard} this matrix is $A$, so that this setup is very limiting. In this section, this is the matrix $AX_{\child(i)}$ and therefore at the digression of the teacher and to a large extend independent on the problem matrix $A$.
\end{remark}

\begin{remark} \label{remark:meta-heuristic}

  The sparse factorization in algorithm \ref{alg:train} condenses the knowledge $X_{\child(i)}$ into $X_i$, allowing more sparse $z_i$ than $z_{\child(i)}$ and as a consequence to tackle more difficult, or less sparse, problems $x$. This condensation is crucial to progress in the curriculum, but is in itself a meta-heuristic to consolidate knowledge.   It is comparable to Occam's razor and the human preference for simple solutions. More flexible meta-heuristics are left for future research.

\end{remark}

\subsection{Learnable Trees}

The algorithm of the last section is summarized in Algorithm \ref{alg:train-tree}. All assumptions together with some technical ones are contained in the following definition.

\begin{definition}
  \label{def:tree-train}
  We call a tree of problem classes $\class_i$, $i \in \indexset$ \emph{learnable} if
  \begin{enumerate}
    \item $X_i = X_{\child(i)} W_{\child(i)}$ for all $j \in \child(i)$, where $X_i$ has $s_i$ sparse columns and $W_{\child(i)}$ has $t/\ttrain \ge 1$ sparse columns so that $s_i \ttrain \le s_j t$.
    \item Each node has at most $\nc$ children.
    \item For each tree node $i$, the matrix $X_i$ has full column rank.
    \item For all tree nodes $i$ the matrix product $A [\textalg{Scale}(X_{\child(i)})]$ satisfies the null space property of order $\sqrt{2}t$.
  \end{enumerate}
  In addition we have the following implementations
  \begin{enumerate}[resume*]
    \item On each tree node, we have implementations of \textalg{Scale}.
    \item We have a solver \textalg{SolveL} for the leave nodes, satisfying Assumption \ref{assumption:enough-simple}.
  \end{enumerate}
  The teacher generates learning problems according to
  \begin{enumerate}[resume*]
    \item On each node $i$, the sampling of training problems satisfies Assumption \ref{assumption:sparse-factorization-alternative} with $X = X_i$.
  \end{enumerate}
\end{definition}

As reasoned above, we obtain the following learning guarantees. For a formal proof, see Appendix \ref{appendix:tree-train}.

\newcommand{\PropTreeTrain}{
  Let $\class_i$, $i \in \indexset$ be learnable according to Definition \ref{def:tree-train}. Then, there exits an implementation of \textalg{SparseFactor} and constants $c > 0$ and $C \ge 0$ independent of the probability model, dimensions and sparsity, so that with probability at least
  \[
    1 - C \nc s_0^{\frac{\log \nc}{\log (c_s t/\ttrain)}} p^{-c}
  \]
  the output $\outX_i = \textalg{TreeTrain}(\class_i)$ of Algorithm \ref{alg:train-tree} is a scaled permutation permutation $\textalg{Scale}(\outX_i) = \textalg{Scale}(X_i P)$ for some permutation matrix $P$.
}

\begin{proposition} \label{prop:tree-train}
  \PropTreeTrain
\end{proposition}

\begin{remark} \label{remark:not-necessarily-global-min}
  The results states that we can recover the root node up to permutation and scaling. It is not strictly required that the solutions in the corresponding class $\class_i$ are global $\ell_0$ minimizers, although, of course, this is the intended use case. This is ensured separately in the applications in Sections \ref{sec:global-min-cols} and \ref{sec:global-min-samples}.
\end{remark}

The biggest problem with learning hard problems $\class_{<t}$ from easy problems $\classE$ in Theorem \ref{th:train} is the need for a solver for the easy problems, as discussed in Section \ref{sec:easy-hard}. The hierarchical structure of Proposition \ref{prop:tree-train} completely eradicates this assumption, except for the leave nodes, which ideally have sparsity $\mathcal{O}(1)$ so that brute force solvers are a viable option.

\begin{algorithm}
  \begin{algorithmic}
    \State $\textalg{Solve}_X$: Solve the modified $\ell_1$-minimization \eqref{eq:l1-with-prior} with the given matrix $X$
    \State $\textalg{SolveL}$: Solver for leave nodes.
    \State $\textalg{Train}(A, b_1, \dots, b_q, \textalg{Solve})$: Algorithm \ref{alg:train} using the given solver subroutine.
    \State
    \Function{TreeTrain}{class $\class_i$}
      \State Get matrix $A$ and training samples $b_1, \dots, b_q$ from teacher.
      \If{$\class_i$ has children}
        \State Compute $X_j = \textalg{TreeTrain}(\class_j)$ for $j \in \child(i)$
        \State Concatenate all child matrices $X = [X_j]_{j \in \child(i)}$
        \State \Return $X_i = \textalg{Train}(A, b_1, \dots, b_q, \textalg{Solve}_X)$
      \ElsIf{$\class_i$ has no children}
        \State \Return $X_i = \textalg{Train}(A, b_1, \dots, b_q, \textalg{SolveL})$
      \EndIf
    \EndFunction
  \end{algorithmic}
  \caption{Tree training}
  \label{alg:train-tree}
\end{algorithm}

\subsection{Cost}
\label{sec:cost}

Let us consider the cost of learnable trees from Definition \ref{def:tree-train}. The number of nodes grows exponentially in the depth of the tree, but the depth only grows logarithmically with regard to the sparsity $s_0$ of the root node, given that we advance the sparsities $s_i$ as fast as \eqref{eq:sparsity-constraints} allows.

\newcommand{\lemmaNNodes}{
  Let $s_0$ be the sparsity of the root node of the tree. Assume that each node of the tree has at most $\nc$ children and that $s_i \ttrain \gtrsim c s_j t$ for $c \ge 0$ and all $j \in \child(i)$. Then the tree has at most
  \[
    \nc^{N+1} = \nc s_0^{\frac{\log \nc}{\log (c t/\ttrain)}}
  \]
  nodes.
}
\begin{lemma} \label{lemma:n-nodes}
  \lemmaNNodes
\end{lemma}

The proof is given in Appendix \ref{appendix:n-nodes}. Since on each node, the number of training samples and the runtime of the training algorithm are both polynomial, this lemma ensures that the entire curriculum is learned in polynomial time, with an exponent depending on $\gamma$, and the ratio $t/\ttrain$.

\section{A tree Construction}
\label{sec:model-tree}

Definition \ref{def:tree-train} and Proposition \ref{prop:tree-train} state several conditions on classes $\class_i$ and their matrices $X_i$ that allow the student to successfully learn the entire tree. While most are relatively simple dimensional requirements, the most severe is the NSP condition on $A [\textalg{Scale}(X_{\child(i)})]$. By \cite{KasiviswanathanRudelson2019} or Theorem \ref{th:train-rip} this is expected for random $X_i$ and for a more realistic model scenario, we add in a deterministic component.

The deterministic part guarantees that every global $\ell_0$-minimizer
\begin{equation} \label{eq:l0-min-classes}
\begin{aligned}
& \min_{x \in \real^n} \|x\|_0, &
& \text{s.t.}
& Ax & = b
\end{aligned}
\end{equation}
can be learned, for arbitrary right hand side $b$ and only minor rank assumptions on $A$. The random part is used as a model for further solutions in class. While this model class may not be fully realistic, it provides a proof of principle that is a little better than the deterministic or random parts alone.

\begin{remark}
  The model shall demonstrate that learning of any deterministic problem is possible, but is is not intended as a practical curriculum design.
\end{remark}

\subsection{Tree Result}

Given $A$ and $x$, we construct a partially random learnable tree whose root class contains $x$. To this end, we first partition the support $\supp(x)$ into non-overlapping patches $\{J_1, \dots, J_{q}\} = \suppx$ and then $x$ into corresponding pieces contained in the columns of the matrix
\begin{equation} \label{eq:x-split-tree}
  S_{jl} := \left\{ \begin{array}{ll}
    x_j & j \in J_l \\
    0   & \text{else}.
  \end{array} \right.
\end{equation}
The columns are spread into the classes of the following learnable tree, with condition number $\kappa(\cdot)$.

\newcommand{\PropClassTree}[1]{
  Let $A \in \real^{m \times n}$ and split $x \in \real^n$ into $q=2^L$, $L \ge 1$ components $S$ given by \eqref{eq:x-split-tree}. If
  \begin{enumerate}
    \item #1{item:class-tree-1} $AS$ has full column rank.
    \item On each tree node, we have implementations of \textalg{Scale}.
    \item \textalg{SolveL} satisfies Assumption \ref{assumption:enough-simple} on the leave nodes.
    \item
    \begin{align} #1{eq:tree-t-bound}
      t & \gtrsim \log p^2 + \log^3 p, &
      1 \lesssim t \lesssim \sqrt{p}
    \end{align}
    \item
    \begin{equation} #1{eq:tree-rip}
        \min_{J \in \suppx} \frac{\norm{A_{\cdot J}}_F^2}{\norm{A_{\cdot J}}^2}
        \gtrsim t \kappa(AS) L
        + t \kappa(AS) \log \frac{cp}{t}
    \end{equation}
  \end{enumerate}
  for some generic constant $c$, with probability at least
  \[
      1 - 2 \exp \left(
        -c \frac{1}{\kappa(AS)} \min_{J \in \suppx} \frac{\norm{A_{\cdot J}}_F^2}{\norm{A_{\cdot J}}^2}
      \right)
  \]
  there is a learnable binary tree of problem classes $\class_i$, $i \in \indexset$ of depth $L$, given by matrices $X_i$ and sparsity $t$ so that
  \begin{enumerate}
    \item The root class $i$ contains $x \in \class_i$.
    \item $t / \ttrain = 2$.
    \item Each class' matrix $X_i$ contains $p$ columns, consisting of a piece of $X$ and otherwise random entries (dependent between classes).
  \end{enumerate}
}

\begin{proposition} \label{prop:class-tree}
  \PropClassTree{\label}{$x$ is a global minimizer of \eqref{eq:l0-min-classes}.}
\end{proposition}

By the following Lemma, proven in Appendix \ref{appendix:l0-min-split}, the first Assumption \ref{item:class-tree-1} is automatically satisfied for global $\ell_0$ optimizers.

\newcommand{\LemmaLZeroMinSplitIndependent}{
  Assume the columns of $S \in \real^{n \times q}$ have non-overlapping support and $z \in \real^q$ with non-zero entries. If the vector $x = S z$ is the solution of the $\ell_0$-minimization problem \ref{eq:l0-min-classes}, then the columns of $AS$ are linearly independent.
}

\begin{lemma} \label{lemma:l0-min-split-independent}
  \LemmaLZeroMinSplitIndependent
\end{lemma}

For possible implementations of \textalg{SolveL}, let us estimate the sparsity at the leave nodes. Since $\min_{J \in \suppx} \frac{\norm{A_{\cdot J}}_F^2}{\norm{A_{\cdot J}}^2} \le |J|$, in the most favorable case $\min_{J \in \suppx} \frac{\norm{A_{\cdot J}}_F^2}{\norm{A_{\cdot J}}^2} \sim |J|$ and for $t$ as small as possible, the condition \eqref{eq:tree-rip} reduces to
\begin{equation} \label{eq:support-estimate-1}
  |J| \gtrsim L t + t \log p \gtrsim L \log p + (\log p)^2,
\end{equation}
posing a limit on the minimal support size we can achieve at the leaves of the tree.
In order to eliminate $L$, let us assume that all $J$ are of equal size and set $s = \norm{x}_0$. Since the tree has $2^L$ leaves, this implies that $s = |J| 2^L$ and thus $\log s = \log|J| + L \ge L$. Thus, condition \eqref{eq:support-estimate-1} reduces to
\begin{equation*}
  |J| \gtrsim \log s \log p + (\log p)^2.
\end{equation*}
Hence, on the leave nodes, a brute force \textalg{SolveL} search of $|J|$ sparse solutions, considers about $n^{|J|} \ge n^{\log s}$ possible supports. While significantly better that $n^s$ possible supports for finding $x$ directly, the former number is not of polynomial size.
In order to drive down the search size to $\mathcal{O}(1)$, we can iterate the tree construction for every column in every leave node. As we see in the next section this leads to a total tree of polynomial size.

\subsection{Tree Extension}
\label{sec:tree-extension}

The curriculum in Proposition \ref{prop:class-tree} shrinks the support size from $s$ to $\log s$. In order to reduce the size further, we may build a new curriculum for every column in every leave $X_i$, if these columns can be split with full rank of $AS$, yielding $p 2^L \le p s$ new curricula. The assumption seems plausible for the random parts and is justified for the deterministic part by the following Lemma, proven in Appendix \ref{appendix:l0-min-split}.

\newcommand{\LemmaLZeroMinSplitGlobal}{
  Assume the columns of $S \in \real^{n \times q}$ have non-overlapping support and $z \in \real^q$ with non-zero entries. If the vector $x = S z$ is the solution of the $\ell_0$-minimization problem \ref{eq:l0-min-classes}, then the columns $S_{\cdot k}$, $k \in [q]$ are global $\ell_0$ optimizers of

  \[
    \begin{aligned}
      S_{\cdot k} & \in \min_{x \in \real^n} \|x\|_0 & & \text{subject to} & A x & = AS_{\cdot k}.
    \end{aligned}
  \]
}

\begin{lemma} \label{lemma:l0-min-split-global}
  \LemmaLZeroMinSplitGlobal
\end{lemma}

\begin{remark}
  The new curricula provide classes that contain columns of leave $X_i$, but not the columns themselves. They must be provided as training samples by the teacher (the right hand side $A(X_i)_{\cdot j}$, not the column $(X_i)_{\cdot j}$). A more careful constructions may reconstruct the columns from combination samples as in learnable trees, which is left for future research.
\end{remark}

Since we aim for leave column support size $|J| \sim 1$ and its lower bound contains $p$, whose size is at our disposal, we shrink it together with the initial (sub-)curriculum support size $s$ by choosing $p \sim s$.
\begin{remark}
  By choosing a large constant or $p \sim s^\alpha$, initially $p$ can be larger than $m$. But by \eqref{eq:support-estimate-1}, towards the leaves $p$ must become small and so that $p \le m$ and the matrix $AX_i$ has more rows that columns. Depending on the kernel of $AX_i$, this may void $\ell_0$ or $\ell_1$-minimization and allow simpler constructions towards the bottom of the tree.
\end{remark}
We iteratively repeat the procedure until the leave support $|J| \sim \mathcal{O}(1)$ is of unit size. The total number $\#(s)$ of required (sub-)curricula for initial support size $s$ satisfies the recursive formula
\[
  \#(s)
  \sim ps \#\left(\log s \log p + (\log p)^2\right)
  \ge s^2 \#\left((\log s)^2\right)
\]
By induction, one easily verifies that $\#(s) \lesssim s^3$, so that we use only a polynomial number of curricula, each of which can be learned in polynomial time. In conclusion, combining all problem classes into one single master tree, {\bf this yields a curriculum for a student to learn the root $\class_0$ in polynomial time, including a predetermined solution $x$.} The problem classes can be fairly large at the top of the tree and must be small at the leaves. At the breaks between different curricula, the training samples must be of unit size containing only one column of the next tree.

\subsection{Construction Idea}
\label{sec:node-idea}

In Proposition \ref{prop:class-tree}, all class matrices $X_i$ are derived from the single matrix
\begin{equation*}
  X := S Z^T + D R (I - Z Z^T).
\end{equation*}
The first summand is the deterministic part, with components $S$ of $x$ defined in \eqref{eq:x-split-tree} and matrix $Z$ with orthogonal columns that ensures correct dimensions. The second summand is the random part with random matrix $R$. The projector $(I - Z Z^T)$ ensures that it does not interfere with the deterministic part and $D$ is a scaling matrix to balance both parts.

We choose $Z$ and $R$ so that, upon permutation of rows and columns $X$ is a block matrix
\[
  X = \begin{bmatrix}
    X_1 & &        \\
        & \ddots & \\
	&        & X_q
  \end{bmatrix}
\]
with each block containing one piece $x_J$. The tree is constructed out of these blocks as follows in case $q=4$ and analogously for larger cases.
\begin{center}
  \scalebox{0.75}{\begin{tikzpicture}
    \node {$
      \begin{bmatrix}
        X_1 \\ X_2 \\ X_3 \\ X_4
      \end{bmatrix}
    $} [level distance=1.5cm, sibling distance=5cm]
    child{
      node {$
        \begin{bmatrix}
          X_1 \\ X_2 \\ \\ ~
        \end{bmatrix}
      $} [level distance=2cm, sibling distance=3cm]
      child{
	node {$
          \begin{bmatrix}
            X_1 \\ \\ \\ ~
          \end{bmatrix}
	$}
      }
      child{
	node {$
          \begin{bmatrix}
            \\ X_2 \\ \\ ~
          \end{bmatrix}
        $}
      }
    }
    child {
      node {$
        \begin{bmatrix}
          \\ \\ X_3 \\ X_4
        \end{bmatrix}
      $} [level distance=2cm, sibling distance=3cm]
      child{
	node {$
          \begin{bmatrix}
            \\ \\ X_3 \\ ~
          \end{bmatrix}
	$}
      }
      child{
	node {$
          \begin{bmatrix}
            \\ \\ \\ X_4
          \end{bmatrix}
        $}
      }
    };
  \end{tikzpicture}}
\end{center}
See Appendices \ref{appendix:X} and \ref{appendix:class-tree} for details.

\section{Applications}
\label{sec:3SAT}

\subsection{3SAT and 1-in-3-SAT}

For an example applications, we consider reductions from the $NP$-complete 3SAT and 1-in-3-SAT to sparse linear systems. The paper \cite{AyanzadehHalemFinin2019} considers the other direction. The problems are defined as follows.
\begin{itemize}

  \item \emph{Literal:} boolean variable or its negation, e.g. : $x$ or $\neg x$.

  \item \emph{Clause:} disjunction of one or more literals, e.g.: $x_1 \vee \neg x_2 \vee x_3$.

  \item \emph{3SAT:} satisfiability of conjunctions of clauses with three literals. For a positive result, at least one literal in each clause must be true.

  \item \emph{1-in-3-SAT:} As 3SAT, but for a positive result, exactly one literal in each clause must be true.

\end{itemize}
Both problems are $NP$-complete an can easily be transformed into each other. In this section, we reduce a 1-in-3-SAT problem with clauses $c_k$, $k \in [m]$ and boolean variables $x_i$, $i \in [n]$ to a sparse linear system, following techniques from \cite{GeJiangYe2011}. For each boolean variable $x_i$, we introduce two variables $y_i \in \real$ corresponding to $x_i$ and $z_i \in \real$ corresponding to $\neg x_i$ for $i \in [n]$. For each clause $c_k$, we define a pair of vectors $C_k, \, D_k$. The vector $C_k$ has a one in each entry $i$ for which the corresponding literal (not variable) $x_i$ is contained in the clause $c_k$ and likewise $D_k$ has a one in each entry $i$ for which the literal $\neg x_i$ is contained in $c_k$. All other entries of $C_k$ and $D_k$ are zero. It is easy to see that
\begin{multline} \label{eq:1-in-3-SAT-row}
  \text{$y \in \{0,1\}^n$ and $z_i = \neg y_i$} \\
  \text{$\Rightarrow$ Exactly one literal in $c_k$ is true if and only if $C_k^T y + D_k^T z = 1$.}
\end{multline}
We combine the linear conditions into the linear system
\begin{align} \label{eq:SAT-A}
  A & :=
  \begin{bmatrix}
    \cdots & C_1^T   & \cdots & \cdots & D_1^T   & \cdots \\
           & \vdots  &        &        & \vdots  &        \\
    \cdots & C_m^T   & \cdots & \cdots & D_m^T   & \cdots \\
    \ddots &         &        & \ddots &         &        \\
           & I_{nn} &        &        & I_{nn} &        \\
           &         & \ddots &        &         & \ddots \\
  \end{bmatrix},
  &
  b & :=
  \begin{bmatrix}
    1 \\ \vdots \\ 1 \\ 1 \\ \vdots \\ 1 \vdots
  \end{bmatrix}
\end{align}
together with some identity blocks that together with the $\ell_0$-minimization
\begin{equation} \label{eq:1-in-3-SAT-system}
  \begin{aligned}
    & \min_{y,z \in \real^n} \|y\|_0 + \|z\|_0 &
    & \text{subject to} &
    A \begin{bmatrix} y \\ z \end{bmatrix} = b.
  \end{aligned}
\end{equation}
ensure that $y \in \{0,1\}^n$, when possible.

\begin{lemma} \label{lemma:1-in-3-SAT-system}
  The clauses $c_k$ corresponding to $C_k$ and $D_k$, $k \in [m]$ are 1-in-3 satisfiable if and only if \eqref{eq:1-in-3-SAT-system} has a $n$ sparse solution.
\end{lemma}

\begin{proof}

The $i$-th row of the identity blocks is $y_i + z_i = 1$. The solution is either $2$-sparse or $1$-sparse with $y_i=1, \, z_i=0$ or $y_i=0,\,z_i=1$. The latter two cases are true for all $i$ if and only if $y$ and $z$ combined are $n$ sparse. Then the result follows from \eqref{eq:1-in-3-SAT-row}.

\end{proof}

\subsection{Model Class}

The 1-in-3-SAT reduction is not suitable for our curriculum learning because the solutions have non-negative entries and therefore cannot be the result of a mean-zero random sampling, required for RIP properties. Therefore, we consider the following larger class
\begin{align*}
  A & = \begin{bmatrix}
    A_{11} & A_{12} \\
    I_{n/2} & I_{n/2}
  \end{bmatrix}
  \in \real^{m \times n}, &
  b & = \begin{bmatrix}
    b_1 \\ b_2
  \end{bmatrix}
  \in \real^n
\end{align*}
for two sparse matrices $A_{1j} \in \{0,1\}^{(m-n/2) \times (n/2)}$ and arbitrary solution vectors $x \in \real^n$. As in Lemma \ref{lemma:1-in-3-SAT-system}, the two identity blocks ensure that any solution $x$ of $Ax = b$ must have support at least $\|x\|_0 \ge \|b_2\|_0$. In the 1-in-3-SAT case, equality corresponds to satisfiable problems. Likewise, we ensure that all training problems satisfy $\|x\|_0 = \|b_2\|_0$, which automatically implies that they are global $\ell_0$ optimizers.

\begin{remark} \label{remark:globall-l0-minimizer}
  If $\|x\|_0 = \|b_2\|_0$, then $x$ is a global $\ell_0$ minimizer.
\end{remark}

\subsection{Curricula}
\label{sec:curricula}

\subsubsection{Curriculum I}

We first consider a curriculum of Proposition \ref{prop:class-tree}, as shown in Figure \ref{fig:curriculum-theory}. The $*$ entries are mean-zero random $\pm 1$ and the $x$ entries are random $\{0,1\}$. The latter have non-zero mean, which is not amenable to RIP conditions and used as a model for the deterministic part of the theory. In all experiments, \textalg{Scale} is implemented by snapping the output of \textalg{SparseFactor} to the discrete values $\{-1, 0, 1\}$, which allows exact recovery of all nodes $X_i$, without numerical errors.

Formally, the curriculum satisfies the construction \ref{item:model-first} -- \ref{item:model-last} in the proof of Proposition \ref{prop:class-tree} with the index sets
\begin{align*}
  & \Big[\underbrace{1, \dots, |J|}_{J_1}, \quad \dots \quad, \underbrace{n-|J|, \dots, n}_{J_q}\Big], &
  & \Big[\underbrace{1, \dots, |K|}_{K_1}, \quad \dots \quad, \underbrace{p-|K|, \dots, p}_{K_q}\Big]
\end{align*}
and $Z = \begin{bmatrix}e_1 & e_{|K|+1} & e_{2|K|+1} & \dots \end{bmatrix}$ with unit basis vectors $e_k$ for the first index in each block $K_i$.

\begin{figure}[htb]
  \begin{center}
    \scalebox{0.75}{\begin{tikzpicture}
      \node {$
        \begin{bmatrix}
          x & * & \dots & * \\
          x & * & \dots & * \\
          x & * & \dots & * \\
          x & * & \dots & * \\
        \end{bmatrix}
      $} [level distance=2.5cm, sibling distance=5cm]
      child{
        node {$
          \begin{bmatrix}
            x & * & \dots & * \\
            x & * & \dots & * \\
              &   &       &   \\
              &   &       &   \\
          \end{bmatrix}
        $} [level distance=2cm, sibling distance=3cm]
        child{node {$\vdots$}}
        child{node {$\vdots$}}
      }
      child {
        node {$
          \begin{bmatrix}
              &   &       &   \\
              &   &       &   \\
            x & * & \dots & * \\
            x & * & \dots & * \\
          \end{bmatrix}
        $} [level distance=2cm, sibling distance=3cm]
        child{node {$\vdots$}}
        child{node {$\vdots$}}
      };
    \end{tikzpicture}}
  \end{center}
  \caption{$X_i$ matrices for a curriculum \ref{item:model-first} -- \ref{item:model-last} and Proposition \ref{prop:class-tree}. $x$ can be different in each row and $*$ are random entries.}
  \label{fig:curriculum-theory}
\end{figure}
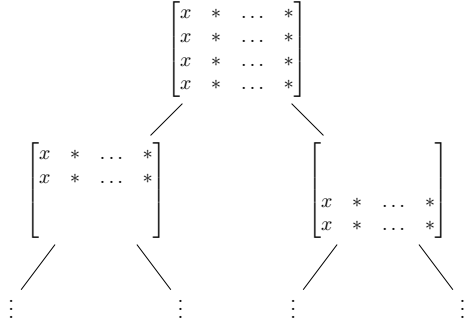

\subsubsection{Curriculum II}
\label{sec:global-min-cols}

For none of the solutions in the problem classes in Curriculum I we know if they are global $\ell_0$ minimizers. While this is not necessarily an issue for the tree construction, as outlined in Remark \ref{remark:not-necessarily-global-min}, it is not fully satisfactory and global minimizers can be obtained as follows. First, we split the columns according to the identity blocks in $A$, as shown in Figure \ref{fig:curriculum-l0-min-cols}. Each component in the upper block $y$ or $*$, has exactly on corresponding component in the lower block $z$ or $+$ so that for each pair at most one entry is non-zero. As a result each column has the required sparsity to guarantee that it is a global $\ell_0$ minimum by Remark \ref{remark:globall-l0-minimizer}.

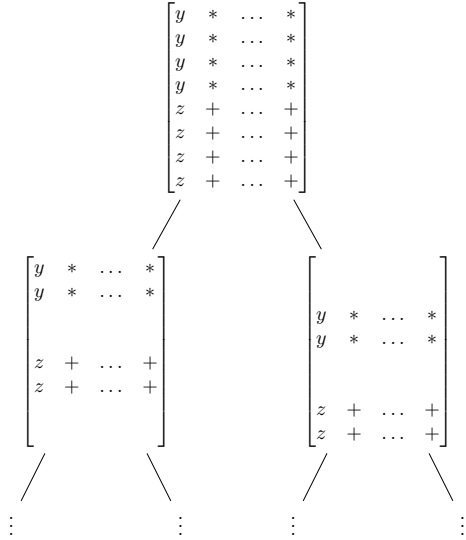
\begin{figure}[htb]
  \begin{center}
    \scalebox{0.75}{\begin{tikzpicture}
      \node {$
        \begin{bmatrix}
          y & * & \dots & * \\
          y & * & \dots & * \\
          y & * & \dots & * \\
          y & * & \dots & * \\
          z & + & \dots & + \\
          z & + & \dots & + \\
          z & + & \dots & + \\
          z & + & \dots & + \\
        \end{bmatrix}
      $} [level distance=4.5cm, sibling distance=5cm]
      child{
        node {$
          \begin{bmatrix}
            y & * & \dots & * \\
            y & * & \dots & * \\
              &   &       &   \\
              &   &       &   \\
            z & + & \dots & + \\
            z & + & \dots & + \\
              &   &       &   \\
              &   &       &   \\
          \end{bmatrix}
        $} [level distance=3cm, sibling distance=3cm]
        child{node {$\vdots$}}
        child{node {$\vdots$}}
      }
      child {
        node {$
          \begin{bmatrix}
              &   &       &   \\
              &   &       &   \\
            y & * & \dots & * \\
            y & * & \dots & * \\
              &   &       &   \\
              &   &       &   \\
            z & + & \dots & + \\
            z & + & \dots & + \\
          \end{bmatrix}
        $} [level distance=3cm, sibling distance=3cm]
        child{node {$\vdots$}}
        child{node {$\vdots$}}
      };
    \end{tikzpicture}}
  \end{center}
  \caption{$X_i$ matrices for a curriculum with $\ell_0$ minimal columns.}
  \label{fig:curriculum-l0-min-cols}
\end{figure}

\subsubsection{Curriculum III}
\label{sec:global-min-samples}

In Curriculum II the columns are global $\ell_0$ minimizers, but their linear combinations in the classes $\class_i$ or the training samples are generally not, which can be fixed by the modification in Figure \ref{fig:curriculum-l0-min}. All blocks individually work as before, but instead of allowing all possible sparse linear combinations of the columns, we only allow one non-zero contribution from each block column. This ensures the sparsity requirements in Remark \ref{remark:globall-l0-minimizer} so that all problems in class are global $\ell_0$ minimizers.

Since the $y$ and $z$ entries are non-negative, this allows us to build a curriculum for arbitrary 1-in-3-SAT problems. However, the curriculum is in the larger signed problem class. If we can build a curriculum in 1-in-3-SAT itself remains open.

\begin{figure}[htb]
  \begin{center}
    \scalebox{0.75}{\begin{tikzpicture}
      \node {$
        \begin{bmatrix}
          y & * &   &   &        \\
          y & * &   &   &        \\
            &   & y & * &        \\
            &   & y & * &        \\
            &   &   &   & \ddots \\
          z & + &   &   &        \\
          z & + &   &   &        \\
            &   & z & + &        \\
            &   & z & + &        \\
            &   &   &   & \ddots \\
        \end{bmatrix}
      $} [level distance=5.5cm, sibling distance=5cm]
      child{
        node {$
          \begin{bmatrix}
            y & * &   &   &        \\
              &   &   &   &        \\
              &   & y & * &        \\
              &   &   &   &        \\
              &   &   &   & \ddots \\
            z & + &   &   &        \\
              &   &   &   &        \\
              &   & z & + &        \\
              &   &   &   &        \\
              &   &   &   & \ddots \\
          \end{bmatrix}
        $} [level distance=3.5cm, sibling distance=3cm]
        child{node {$\vdots$}}
        child{node {$\vdots$}}
      }
      child {
        node {$
          \begin{bmatrix}
              &   &   &   &        \\
            y & * &   &   &        \\
              &   &   &   &        \\
              &   & y & * &        \\
              &   &   &   & \ddots \\
              &   &   &   &        \\
            z & + &   &   &        \\
              &   &   &   &        \\
              &   & z & + &        \\
              &   &   &   & \ddots \\
          \end{bmatrix}
        $} [level distance=3.5cm, sibling distance=3cm]
        child{node {$\vdots$}}
        child{node {$\vdots$}}
      };
    \end{tikzpicture}}
  \end{center}
  \caption{$X_i$ matrices for a curriculum with $\ell_0$ minimal columns.}
  \label{fig:curriculum-l0-min}
\end{figure}

\subsection{Numerical Experiments}
\label{sec:experiments}

Table \ref{table:results} contains results for Curricula II and III. All $\ell_1$-minimizations problems are solved by gradient descent in the kernel of $Ax=b$ and the sparse factorization is implemented by $\ell_4$-maximization \cite{ZhaiYangLiaoEtAl2020}. Solutions on the leave nodes are given instead of brute force solved. As in \cite{Welper2021}, Algorithm \ref{alg:train} contains an additional grader that sorts out wrong solutions from \textalg{Solve}, which often depend on the gradient descent accuracy.
\begin{itemize}
  \item \emph{Curriculum II:} We train three tree nodes on two levels. Grader tests to accuracy $10^{-4}$. The results are the average of $5$ independent runs.
  \item \emph{Curriculum III:} We train one tree node. The training sample matrices \eqref{eq:SAT-A} are preconditioned per node, not globally as in Proposition \ref{prop:class-tree}, below. Grader tests to accuracy $10^{-3}$. The results are the average of $2$ independent runs.
\end{itemize}
Table \ref{table:results} contains the results. It includes average ranks to show that the systems $AX$ are non-trivial with non-zero kernel and the row $\% \textalg{Validate}$ shows the percentage of correctly recovered training samples according to the grader. A major bottleneck is the number of training samples for each node, which scales log quadratically for $\ell_4$ maximization (but only log linear for unique factorization without algorithm \cite{SpielmanWangWright2012}). The last line shows that in the majority of cases we can recover the tree nodes $X_i$. The misses depend on solver parameters as e.g. iteration numbers and the size of random matrices.

%
%
%
\begin{table}
  \begin{center}
\begin{tabular}{llll}
\toprule
{} & \multicolumn{2}{l}{Curr. I} & Curr. II \\
Depth &       0 &      1 &        0 \\
\midrule
$m$                               &      96 &     96 &      121 \\
$n$                               &     128 &    128 &      162 \\
$p\left(X_{\child(i)}\right)$     &     102 &    102 &      459 \\
Rank$\left(AX_{\child(i)}\right)$ &   96.00 &  62.80 &   113.00 \\
$\#$ Samples                      &   10000 &  10000 &    90000 \\
\% \textalg{Validate}            &    0.55 &   0.91 &     0.98 \\
$\#(X_{student} = X)$             &     5/5 &   7/10 &      2/2 \\
\bottomrule
\end{tabular}

   \end{center}
  \caption{Results of numerical experiments, Section \ref{sec:experiments}, averaged over all runs and all nodes of given depth. The second but last row shows the percentage of successful training solutions, according to the grader. The last row shows the number of successfully recovered $X_i$ for the given level out of the total number of trials.}
  \label{table:results}
\end{table}

\section{Conclusion}

Although sparse solutions of linear systems are generally hard to compute, many subclasses are tractable. In particular, the prior knowledge $x = Xz$ with sparse $z$ allows us to solve problems with only mild assumptions on $A$. We learn $X$ from a curriculum of easy samples and condensation of knowledge at every tree node. The problems in each class must be compatible so that $AX$ satisfies the null space property. To demonstrate the feasibility of the approach, we show that the algorithms can learn a class $X$ of non-trivial size that contains an arbitrary solution $x$.

The results provide a rigorous mathematical model for some hypothetical principles in human reasoning, including expert knowledge and its training in a curriculum. To be applicable in practice, further research is required, e.g.:

\begin{itemize}
  \item The mapping of SAT type problems into sparse linear problems lacks several invariances, e.g. a simple reordering of terms may invalidate acquired knowledge. The problem is similar to feature engineering in machine learning.
  \item For sparse factorization, the required number of samples scales quadratically, up to a log factor, which is the biggest computational bottleneck in the numerical experiments.
  \item The curriculum is designed so that knowledge can be condensed by sparse factorization, which in itself is a meta-heuristic. One may need to dynamically adapt the condensation heuristic to real data. Since sparse factorization algorithms themselves often rely on $\ell_1$ minimization, similar approaches as discussed in the paper are conceivable.
  \item Not all knowledge can be combined into one class $X$ so that $AX$ satisfies the null space property. Hence, instead of a knowledge tree with one root node, a more flexible knowledge graph, together with a decision criterion which node to use for a given problem, seems more plausible.
\end{itemize}

\appendix

\bibliographystyle{abbrv}
\bibliography{cstree}

\section{Details and Proofs}

\subsection{Easy and Hard Problems: Theorems \ref{th:train}, \ref{th:train-rip}}
\label{appendix:easy-hard}

Theorem \ref{th:train} contains some small changes to the original reference \cite{Welper2021}. In the original version \ref{assumption:sparse-factorization-alternative} contains two extra inequalities
\begin{align*}
  n & \ge \bar{c}_1 p \log p, &
  \frac{1}{p} & \le \frac{s}{n} \le \bar{c}_2,
\end{align*}
which are used to ensure that $X$ has full rank \cite[Proof of Theorem 4.2 with (A3), Item 4]{Welper2021}. We assume this directly in \ref{assumption:rank} and leave out the inequalities.

For Theorem \ref{th:train-rip}, the reference \cite{Welper2021} requires the extra assumption that $Ax=b$ has unique $st$ sparse solutions, which is only used to verify that solutions of \textalg{Solve} are correct. In our case, this is implicitly contained in \ref{assumption:enough-simple}, instead.

\subsection{Tree Size: Lemma \ref{lemma:n-nodes}}
\label{appendix:n-nodes}

\begin{lemma}[Lemma \ref{lemma:n-nodes} restated]
  \lemmaNNodes
\end{lemma}

\begin{proof}

Let $\ell_i$ be the level of a node, i.e. the distance to the root node, and $\nn$ the maximal level of all nodes. Each level has at most $\nc^{\nn-i}$ nodes and thus the full tree has at most
\[
  \sum_{i=0}^\nn \nc^{\nn-i} = \frac{\nc^{\nn+1} - 1}{\nc - 1} \le \nc \nc^\nn
\]
nodes.

It remains to estimate $\nn$. By induction on the assumption $s_i \ttrain \ge c s_j t$ we have
\[
  s_j
  \le \left( \frac{\ttrain}{ct} \right)^{\ell_j} s_0
\]
and thus, since necessarily $s_j \ge 1$, we conclude that
\[
  s_0
  \ge \left( \frac{ct}{\ttrain} \right)^\nn.
\]
Plugging in $\nc^\nn = \left(\frac{ct}{\ttrain}\right)^{\nn \frac{\log \nc}{\log ct/\ttrain}}$ the number of nodes is bounded by
\[
  \nc \nc^N
  = \nc \left(\frac{ct}{\ttrain}\right)^{\nn \frac{\log \nc}{\log ct/\ttrain}}
  \le \nc s_0^{\frac{\log \nc}{\log ct / \ttrain}}.
\]
\end{proof}

\subsection{Learnable Trees: Proposition \ref{prop:tree-train}}
\label{appendix:tree-train}

\begin{proposition}[Proposition \ref{prop:tree-train} restated]
  \PropTreeTrain
\end{proposition}

\begin{proof}

The result follows from inductively applying Theorem \ref{th:train} on each node of the tree, starting at its leaves. The assumptions of Theorem \ref{th:train} are easily matched with the given ones, except for \ref{assumption:enough-simple}, which we verify separately for leave and non-leave nodes.

\begin{enumerate}

  \item \emph{Leave Nodes:} For the leave nodes \ref{assumption:enough-simple} is assumed. This is required because the globally sparsest solution of $Ax=b$ may not be unique, in which case \ref{assumption:enough-simple} ensures that we pick an in class solution.

  \item \emph{Non-Leave Nodes:} Let $z$ be a column of the training sample $Z$ and $x = X_i z$. By \eqref{eq:X-child}, we have
  \[
    x = X_i z = X_{\child(i)} W_{\child(i)} z =: X_{\child(i)} w
  \]
  with $t$ sparse $w$ because $W_{\child(i)}$ has $t/\ttrain$ sparse columns and $z$ is $\sqrt{2}\ttrain$ sparse, with probability at least $1-2p^{-c}$ (see the proof of Theorem \ref{th:train}, Item 2, in \cite{Welper2021}). Since $A X_{\child(i)}$ satisfies the $\sqrt{2}t$-RIP, the correct solution $x$ is recovered by the modified $\ell_1$-minimization \eqref{eq:l1-with-prior} and hence by $\textalg{Solve}_{X_i}$.

\end{enumerate}

Finally, we add up the probabilities. By Theorem \ref{th:train}, the probability of failure on each node is at most $Cp^{-c}$. By Lemma \ref{lemma:n-nodes}, there are at most $\nc s_0^{\frac{\log \nc}{\log (c t/\ttrain)}}$ nodes and thus the result follows from a union bound.

\end{proof}

\subsection{Split of Global \texorpdfstring{$\ell_0$}{l0} Minimizers}
\label{appendix:l0-min-split}

This section contains two lemmas that state the splits of $\ell_0$ minimizers are again $\ell_0$ minimizers and that they are linearly independent.

\begin{lemma}[Lemma \ref{lemma:l0-min-split-independent} restated] \label{appendix:lemma:l0-min-split-independent}
  \LemmaLZeroMinSplitIndependent
\end{lemma}

\begin{proof}

Let $x_i$ be the columns of $S$ and assume that the $Ax_i$, $i \in [t]$ are linearly dependent. Then there exists a non-zero $y \in \real^t$ such that $\sum_{i=1}^t A x_i y_i = 0$. Without loss of generality, let $y_1 \ne 0$ so that
\[
  A x_1 = - A \sum_{i=2}^t x_i \frac{y_i}{y_1}.
\]
We use this identity to eliminate $x_1$:
\begin{equation*}
  b = Ax
   = A \sum_{i=1}^t x_i z_i,
  = A x_1 z_1 + A \sum_{i=2}^t x_i z_i,
  = A \sum_{i=2}^t x_i z_i \left( 1 - \frac{y_i}{y_1} z_0 \right)
  =: A \bar{x}.
\end{equation*}
Since all $x_i$ have disjoint support and all $z_i$ are non-zero, we have $\|\bar{x}\|_0 < \|x\|_0$, which contradicts the assumption that $x$ is a $\ell_0$ minimizer and thus all $Ax_i$, $i \in [n]$ must be linearly independent.

\end{proof}

\begin{lemma}[Lemma \ref{lemma:l0-min-split-global} restated]
  \LemmaLZeroMinSplitGlobal
\end{lemma}

\begin{proof}

Assume the statement is wrong. Then for some $k \in [q]$ there is a $y_k$ with
\[
  \begin{aligned}
    \|y_k\|_0 & \le \|S_{\cdot k}\|_0, & A y_k & = AS_{\cdot k}.
  \end{aligned}
\]
Define
\[
  \bar{x}
  := y_k z_k + \sum_{l \ne k} S_{\cdot l} z_l.
\]
Then, we have
\[
  A \bar{x}
  = A y_k z_k + A \sum_{l \ne k} S_{\cdot l} z_l.
  = A \sum_l S_{\cdot l} z_l
  = A S z
  = A x
\]
and since all $S_{\cdot l}$ have disjoint support and $z_l \ne 0$
\[
  \|\bar{x}\|_0
  = \|y_k\|_0 + \sum_{l \ne k} \|S_{\cdot l}\|_0
  < \sum_{l} \|S_{\cdot l}\|_0
  = \|x\|_0.
\]
This contradicts the assumption that $x$ is a global $\ell_0$ minimiser and hence all $S_{\cdot k}$ must be $\ell_0$ minimizers as well.

\end{proof}

\subsection{Tree Nodes for Proposition \ref{prop:class-tree}}

This section contains the construction of the matrices $X$ in the tree nodes used in Proposition \ref{prop:class-tree}.

\subsubsection{Construction of \texorpdfstring{$X$}{X}}
\label{appendix:X}

We follow the idea outlined in Section \ref{sec:node-idea}. For given matrix $A$ and vector $x$, we construct a decomposition matrix $X \in \real^{n \times p}$ and $z$ so that $x = Xz$ for $t$-sparse $z$ and $AX$ satisfies the null space property. The first condition ensures that $x$ is contained in the class $\class_{<t}$ and the second provides solvers \textalg{Solve}. This construction will be used in subsequent sections to define nodes in the curriculum tree. We start with some simple definitions
\begin{enumerate}[label=(M\arabic*)]

  \item \label{item:model-first}By $\splitmat^{m \times n}$ we denote all matrices in $\real^{m \times n}$ whose columns have non-overlapping support.

  \item $\one := \begin{bmatrix} 1 & \cdots & 1 \end{bmatrix}^T$ with dimensions derived from context.

\end{enumerate}
We split $x$ into $q$ non-overlapping components, which we combine into the columns of a matrix $S \in \splitmat^{n \times q}$ so that $x = S\one$. The matrix $S$ has $q$ columns, which is generally less than the $p$ columns we desire for a rich class given by $X$. A convenient way out is to choose some matrix $Z \in \real^{p \times q}$ with orthonormal columns so that $x = S Z^T Z \one = S Z^T z$ with $z := Z\one$. To ensure sparsity of $z$ and for later tree construction, we confine $Z$ to $\splitmat^{p \times q}$.
\begin{enumerate}[resume*]

  \item $S \in \splitmat^{n \times q}$ with non-zero columns.

  \item $Z \in \splitmat^{p \times q}$ with $\ell_2$-normalized columns.

\end{enumerate}
While the matrix $SZ^T$ has the same dimensions as $X$, it is generally low rank and cannot satisfy the NSP. Furthermore, we want a rich class matrix $X$ with further possible random solutions. To this end, we add in a random matrix $R$, but only on blocks of $SZ^T$ that are non-zero to keep sparsity. We define $R$ as follows
\begin{enumerate}[resume*]
  \item Partition the support of $x$ and $[p]$ into disjoint sets
  \begin{align*}
    \suppx & := \{\supp(X_{\cdot l}) : \, l \in [q]\}, &
    \suppz & := \{K_l : \, l \in [q]\}, &
    \supp(Z_{\cdot l}) & \subset K_l, \, l \in [q]
  \end{align*}
  so that each set $J \in \suppx$ corresponds to the support of one component of $x$ in the columns of $S$ and likewise for $Z$. We also need matching pairs $[J,K]$ of blocks
  \[
    \suppxz := \{[\supp(X_{\cdot l}), \supp(Z_{\cdot l})] : \, l \in [q]\},
  \]
  originating form the same respective columns of $S$ and $Z$.

  \item $R \in \real^{n \times p}$ is block matrix
  \begin{align*}
    R_{jk} = \left\{ \begin{array}{ll}
      \text{i.i.d random} & j,k \in [J,K] \in \suppxz \\
      0 & \text{else,}
    \end{array} \right.
  \end{align*}
  whose random entries satisfy
  \begin{align*}
    \E{R_{jk}} & = 0, &
    \E{R_{jk}^2} & = 1, &
    \|R_{jk}\|_{\psi_2} & \le \psitwo
  \end{align*}
  for some constant $\psitwo$ and are absolutely continuous with respect to the Lebesgue measure.
\end{enumerate}
Finally, we need a scaling matrix that will be determined below.
\begin{enumerate}[resume*]
  \item $D \in \real^{n \times n}$ is a diagonal scaling matrix to be determined below.
\end{enumerate}
Then, we define the following class matrix
\begin{enumerate}[resume*]
  \item \label{item:model-last}
  \begin{equation} \label{eq:X-def}
    X := S Z^T + D R (I - Z Z^T),
  \end{equation}
\end{enumerate}
which is random on the kernel of $Z^T$ and matches the previously constructed $SZ^T$ on the orthogonal complement.

The following lemma summarises several elementary properties of the matrices and vectors in \ref{item:model-first} - \ref{item:model-last} that are used in the proofs below. In particular, they satisfy $x = X z$ for $z = Z\one$.

\begin{lemma} \label{lemma:XZ-properties}
For the construction \ref{item:model-first} - \ref{item:model-last} we have:
  \begin{enumerate}
    \item \label{item:Z-orthonormal} $Z^T Z = I$.
    \item $Z Z^T$ is an orthogonal projector.
    \item \label{item:Z-block} Let $\supp(Z_{\cdot l}) \subset K \in \suppz$ for some column $l$. Then
    \[
      (Z Z^T)_{KL} = \left\{ \begin{array}{ll}
	Z_{Kl} Z_{Kl}^T & \text{if }K = L \\
	0 & \text{else.}
      \end{array} \right.
    \]
    \item \label{item:Z-diag} $(Z Z^T)_{KL} = 0$ for all $K \ne L \in \suppz$.
    \item \label{item:Z-diag-projector} $(Z Z^T)_{KK}$ is an orthogonal projector for all $K \in \suppz$.
    \item \label{item:Z-norm-sum}
    For all $u \in \real^p$ we have
    \[
      \sum_{K \in \suppz} \norm{(Z Z^T)_{KK} u_K}^2
      = \norm{Z^T u}^2.
    \]
    \item \label{item:Z-norm-sum-I} For all $u \in \real^p$ we have
    \[
      \sum_{K \in \suppz} \norm{(I - Z Z^T)_{K \cdot} u}^2
      \le \norm{u}^2.
    \]
    \item For $z = Z \one$, we have $ZZ^T z = z$.
    \item For $x = S \one$ and $z = Z \one$, we have $SZ^T z = x$.
    \item For $x = S \one$ and $z = Z \one$, we have $X z = x$.
  \end{enumerate}

\end{lemma}

\begin{proof}

\begin{enumerate}
  \item Since $Z$ is normalized and $Z \in \splitmat^{p \times q}$, all columns are orthonormal.
  \item $ZZ^T$ is symmetric and with Item \ref{item:Z-orthonormal} we have $(Z Z^T) (Z Z^T) = Z (Z^T Z) Z^T = Z Z^T$.
  \item We have
  $
    (Z Z^T)_{KL}
    = \sum_{l=1}^q (Z_{\cdot l} Z_{\cdot l}^T)_{KL}
    = \sum_{l=1}^q Z_{Kl} Z_{Ll}^T,
  $
  which reduces to the formula in the lemma because $K \ne L$ are disjoint and $\supp Z_{\cdot l} \subset K$.
  \item Follows directly from Item \ref{item:Z-block}.
  \item Follows directly from Item \ref{item:Z-block} because the vectors $Z_{Kl}$ is normalized.
  \item For every $K \in \suppz$, let $l \in [q]$ be the corresponding index with $\supp(Z_{\cdot l}) \subset K$. Then, we have
  \begin{multline*}
    \sum_{K \in \suppz} \norm{(Z Z^T)_{KK} u_K}^2
    = \sum_{K,l=1}^q \norm{Z_{Kl} Z_{Kl}^T u_K}^2
    \\
    = \sum_{K,l=1}^q (Z_{Kl}^T u_K)^2
    = \sum_{l=1}^q (Z_{\cdot l}^T u)^2
    = \norm{Z^T u}^2,
  \end{multline*}
  where in the first equality we have used Item \ref{item:Z-block}, in the second that all $Z_{Kl}$ are normalized and in the third that $\supp(Z_{Kl}) \subset K$.
  \item From Item \ref{item:Z-block}, we have
  \[
    (I - Z Z^T)_{K \cdot} u
    = u_K - \sum_{L \in \suppz}(Z Z^T)_{KL} u_L
    = u_K - (Z Z^T)_{KK} u_K.
  \]
  Since by Item \ref{item:Z-diag-projector} the matrix $(I - ZZ^T)_{KK}$ is a projector, it follows that
  \begin{multline*}
    \sum_{K \in \suppz} \norm{(I - Z Z^T)_{K \cdot} u}^2
    = \sum_{K \in \suppz} \norm{(I - Z Z^T)_{KK} u_K}^2
    \\
    \le \sum_{K \in \suppz} \norm{(I - Z Z^T)_{KK}}^2 \norm{u_K}^2
    \le \norm{u}^2.
  \end{multline*}
  \item With Item \ref{item:Z-orthonormal} we have $Z Z^T z = Z Z^T Z \one = Z \one = z$.
  \item With Item \ref{item:Z-orthonormal} we have $S Z^T z = S Z^T Z \one = S \one = x$.
  \item Follows directly from the previous items.
\end{enumerate}

\end{proof}

\subsubsection{Expectation and Concentration}
\label{sec:expectation-concentration}

For the proof of RIP and null space properties, we need expectation and concentration results for $\norm{AXu}$ for an arbitrary $u$.

\begin{lemma} \label{lemma:X-expectation-1}
  Let $u \in \real^p$, $A \in \real^{m \times n}$ and $X$ be the matrix defined in \eqref{eq:X-def}. Then
  \begin{equation*}
    \E{\|AXu\|^2}
    = \left\|A S Z^T u \right\|^2 + \sum_{[J,K] \in \suppxz} \left\|A D_{\cdot J} \right\|_F^2 \left[ \left\|u_K\right\|^2 - \left\|(Z Z^T)_{KK} u_K \right\|^2 \right].
  \end{equation*}
\end{lemma}

\begin{proof}

Since $R$ is zero outside of the blocks $R_{JK}$ for $[J,K] \in \suppxz$, we have
\begin{equation*}
  X u
  = [S Z^T + D R (I - Z Z^T)] u
  = S Z^T u + \sum_{[J,K] \in \suppxz} D_{\cdot J} R_{JK} (I - Z Z^T)_{K \cdot} u
\end{equation*}
and thus
\begin{align*}
  \E{\|AXu\|^2}
  & = \E{\left\|S Z^T u + \sum_{[J,K] \in \suppxz} D_{\cdot J} R_{JK} (I - Z Z^T)_{K \cdot} u \right\|^2}
  \\
  & = \left\|A S Z^T u \right\|^2 + \sum_{[J,K] \in \suppxz} \left\|A D_{\cdot J} R_{JK} (I - Z Z^T)_{K \cdot} u \right\|^2
  \\
  & = \left\|A S Z^T u \right\|^2 + \sum_{[J,K] \in \suppxz} \left\|A D_{\cdot J} \right\|_F^2 \left\|(I - Z Z^T)_{K \cdot} u \right\|^2,
\end{align*}
where in the second line we have used that all blocks $R_{KJ}$ are independent and in the third we have used Lemma \ref{lemma:random-matrix-norm}. We simplify the last term
\begin{align*}
  \left\|(I - Z Z^T)_{K \cdot} u \right\|^2
  & = \left\|u_K - \sum_{L \in \suppz} (Z Z^T)_{KL} u_L \right\|^2
  \\
  & = \left\|u_K - (Z Z^T)_{KK} u_K \right\|^2
  \\
  & = \left\|u_K\right\|^2 - \left\|(Z Z^T)_{KK} u_K \right\|^2,
\end{align*}
where the second and third lines follow from Items \ref{item:Z-diag} and \ref{item:Z-diag-projector} in Lemma \ref{lemma:XZ-properties}, respectively. Hence, we obtain
\begin{equation*}
  \E{\|AXu\|^2}
  = \left\|A S Z^T u \right\|^2 + \sum_{[K,J] \in \suppxz} \left\|A D_{\cdot K} \right\|_F^2 \left[ \left\|u_K\right\|^2 - \left\|(Z Z^T)_{KK} u_K \right\|^2 \right].
\end{equation*}

\end{proof}

If $AS$ has orthonormal columns, we can simplify the expectation. Since this is generally not true, we rename $A \to M$, which will be a preconditioned variant of $A$ later.

\begin{lemma} \label{lemma:X-expectation-2}
  Let $u \in \real^p$ and $M \in \real^{m \times n}$. With $X$, $S$ and $D$ defined in \eqref{eq:X-def}, assume that $MS$ has orthonormal columns and the diagonal scaling is chosen as $D_j = \norm{M_{\cdot J}}_F^{-1}$ for all $j$ in block $J \in \suppx$. Then
  \begin{equation*}
    \E{\norm{MXu}^2}  = \norm{u}^2.
  \end{equation*}
\end{lemma}

\begin{proof}

The result follows from Lemma \ref{lemma:X-expectation-1} after simplifying several terms.
First, since $MS$ has orthonormal columns, we have $(MS)^T (MS) = I$ and thus
\begin{align*}
  \norm{MSZ^Tu}^2
  = u^T Z (MS)^T (MS) Z^T u
  = u^T Z Z^T u
  = \norm{Z^T u}^2.
\end{align*}
Second, for arbitrary $j \in J$, by definition of the scaling $D$, we have
\[
  \norm{MD_{\cdot J}}_F^2
  = \norm{M_{\cdot J}}_F^2 |D_j|^2
  = \norm{M_{\cdot J}}_F^2 \norm{M_{\cdot J}}_F^{-2}
  = 1.
\]
Finally, form Lemma \ref{lemma:XZ-properties} Item \ref{item:Z-norm-sum}, we have
\[
  \sum_{K \in \suppz} \norm{(Z Z^T)_{KK} u_K}^2
  = \norm{Z^T u}^2.
\]
Plugging into Lemma \ref{lemma:X-expectation-1}, we obtain
\begin{align*}
    \E{\norm{MXu}^2}
    & = \norm{M S Z^T u }^2 + \sum_{[J,K] \in \suppxz} \norm{M D_{\cdot J} }_F^2 \left[ \norm{u_K}^2 - \norm{(Z Z^T)_{KK} u_K }^2 \right].
    \\
    & = \norm{Z^T u}^2 + \left( \sum_{[J,K] \in \suppxz} \norm{u_K}^2 \right) - \norm{Z^T u}^2
    \\
    & = \norm{u}^2.
\end{align*}

\end{proof}

Next, we prove concentration inequalities for the random matrix $X$.

\begin{lemma} \label{lemma:X-concentration}
  Let $u \in \real^p$ and $M \in \real^{m \times n}$. With $X$, $S$ and $D$ defined in \eqref{eq:X-def}, assume that $MS$ has orthonormal columns and the diagonal scaling is chosen as $D_j = \norm{M_{\cdot J}}_F^{-1}$ for all $j$ in block $J \in \suppx$. Then
  \begin{equation*}
    \norm{ \norm{MXu}^2 - \norm{u} }_{\psi_2}
    \le C \psitwo^2 \max_{J \in \suppx} \frac{\norm{M_{\cdot J}}}{\norm{M_{\cdot J}}_F} \norm{u}.
  \end{equation*}
\end{lemma}

\begin{proof}

  The result follows from Lemma \ref{lemma:Ax-b-psi2} after we have vectorized $R$. To this end, let $\vecz(\cdot)$ be the vectorization, which identifies a matrix $\real^{a \times b}$ with a vector in $(\real^a) \otimes (\real^b)'$ for any dimensions $a$, $b$. Then, since for all matrices $ABu = (A \otimes u^T) \vecz(B)$, we have
\[
  M D_{\cdot J} R_{JK} (I - (Z Z^T)_{K \cdot} u
  = \left[M D_{\cdot J} \otimes  u^T (I - (Z Z^T)_{K \cdot}^T\right] \vecz\left( R_{JK} \right)
\]
so that
\begin{align*}
  M X u
  & = [M S Z^T + M D R (I - Z Z^T)] u
  \\
  & = M S Z^T u + \sum_{[J,K] \in \suppxz} M D_{\cdot J} R_{JK} (I - Z Z^T)_{K \cdot} u
  \\
  & = M S Z^T u + \sum_{[J,K] \in \suppxz} \left[M D_{\cdot J} \otimes  u^T (I - Z Z^T)_{K \cdot}^T \right] \vecz\left( R_{JK} \right)
  \\
  & =: \vecb + \vecA \vecR,
\end{align*}
with the block matrix and vectors
\begin{align*}
  \vecA & := \left[M D_{\cdot J} \otimes  u^T (I - Z Z^T)_{K \cdot}^T \right]_{[J,K] \in \suppxz} \\
  \vecR & := \left[ \vecz\left( R_{JK} \right) \right]_{[J,K] \in \suppxz} \\
  \vecb & := MSZ^Tu.
\end{align*}
Using Lemma \ref{lemma:Ax-b-expectation} in the fist equality and Lemma \ref{lemma:X-expectation-2} in the last, we have
\[
  \norm{\vecA}_F^2 + \norm{\vecb}^2
  = \E{\norm{\vecA \vecR + \vecb}^2}
  = \E{\norm{MXu}^2}
  = \|u\|^2.
\]
Furthermore, we have
\begin{align*}
  \norm{\vecA}
  & \le \left(\sum_{[J,K] \in \suppxz} \norm{M D_{\cdot J} \otimes  u^T (I - Z Z^T)_{K \cdot}^T}^2 \right)^{1/2}
  \\
  & = \left(\sum_{[J,K] \in \suppxz} \norm{M D_{\cdot J}}^2 \norm{(I - Z Z^T)_{K \cdot}u}^2 \right)^{1/2}
  \\
  & = \max_{J \in \suppx} \norm{M D_{\cdot J}} \left(\sum_{K \in \suppz} \norm{(I - Z Z^T)_{K \cdot}u}^2 \right)^{1/2}
  \\
  & \le \max_{J \in \suppx} \norm{M D_{\cdot J}} \norm{u},
\end{align*}
where in the last inequality we have used Lemma \ref{lemma:XZ-properties}, Item \ref{item:Z-norm-sum-I}. Thus, with Lemma \ref{lemma:Ax-b-psi2}, we have
\begin{multline*}
    \norm{ \norm{MXu} - \norm{u} }_{\psi_2}
    = \norm{ \norm{\vecA \vecR + \vecb} - \left(\norm{\vecA}_F^2 + \norm{\vecb}^2\right)^{1/2} }_{\psi_2}
    \\
    \le C \psitwo^2 \norm{\vecA}
    \le C \psitwo^2 \max_{J \in \suppx} \norm{M D_{\cdot J}} \norm{u}.
\end{multline*}
We can further estimate the right hand side with the definition of diagonal scaling $D$
\[
  \norm{M D_{\cdot J}}
  = \norm{M_{\cdot J} D_{JJ}}
  = \frac{\norm{M_{\cdot J}}}{\norm{M_{\cdot J}}_F},
\]
which completes the proof.

\end{proof}

\subsubsection{RIP of \texorpdfstring{$MX$}{MX}}

We do not show the RIP for $AX$ directly, but for a preconditioned variant. Since we determine the preconditioner later, we first state results for a generic matrix $MX$. With the expectation and concentration inequalities from the previous section, the proof of the RIP is standard, see e.g.
\cite{
BaraniukDavenportDeVoreEtAl2008,
FoucartRauhut2013,
KasiviswanathanRudelson2019%
}.
We first show a technical lemma.

\begin{lemma} \label{lemma:approximate-cover}
  Let $A \in \real^{m \times n}$ and assume that there is a $\frac{\epsilon}{4}$ cover $\mathcal{N} \subset S^{n-1}$ of the unit sphere $S^{n-1}$ with
  \[
    \begin{aligned}
      \left| \norm{A x_i} - 1 \right| & \le \frac{\epsilon}{2} &
      & \text{for all } x_i \in \mathcal{N}.
    \end{aligned}
  \]
  Then
  \[
    \begin{aligned}
      (1-\epsilon) \norm{x} & \le \norm{Ax} \le (1+\epsilon) \norm{x} &
      & \text{for all } x \in \real^n.
    \end{aligned}
  \]

\end{lemma}

\begin{proof}

Let $x \in S^{n-1}$ be the maximizer of the norm so that $\norm{Ax} = \norm{A}$. Then, there is a element $x_i \in \mathcal{N}$ in the cover with $\norm{x-x_i} \le \frac{\epsilon}{4}$ and we obtain the upper bound
\begin{gather*}
  \norm{A}
  = \norm{Ax}
  \le \norm{Ax_i} + \norm{A(x - x_i)}
  \le \norm{Ax_i} + \norm{A} \frac{\epsilon}{4}
  \\ \Rightarrow
  \left(1 - \frac{\epsilon}{4} \right) \norm{A}
  \le \norm{A x_i}
  \\ \Rightarrow
  \norm{A}
  \le \frac{1 + \epsilon/2}{1 - \epsilon/4}
  \le 1 + \epsilon.
\end{gather*}
With the upper bound and the given assumptions, for arbitrary $x \in S^{n-1}$, we estimate the lower bound by
\begin{multline*}
  \norm{Ax}
  \ge \norm{Ax_i} - \norm{A(x - x_i)}
  \ge \norm{Ax_i} - (1+\epsilon) \norm{x - x_i}
  \\
  \ge \left(1 - \frac{\epsilon}{2}\right) - (1+\epsilon) \frac{\epsilon}{4}
  = 1 - \frac{\epsilon}{2} - \frac{\epsilon}{4} - \frac{\epsilon^2}{4}
  \ge 1 - \epsilon.
\end{multline*}
The bounds extend from the sphere to all $x \in \real^n$ by scaling.

\end{proof}

For the following RIP result, we add in an isometry $W \in \real^{p \times \palt}$, with $\norm{W \cdot} = \norm{\cdot}$, which allows us to construct tree nodes $X_i$ from its children by \eqref{eq:X-child} below.

\begin{lemma} \label{lemma:X-rip}
  Let $W \in \real^{p \times \palt}$ be an isometry and for $M \in \real^{m \times n}$, with $X$, $S$ and $D$ defined in \eqref{eq:X-def}, assume that $MS$ has orthonormal columns and the diagonal scaling is chosen as $D_j = \norm{M_{\cdot J}}_F^{-1}$ for all $j$ in block $J \in \suppx$. If
  $
    \min_{J \in \suppx} \frac{\norm{M_{\cdot J}}_F^2}{\norm{M_{\cdot J}}^2}
    \ge \frac{2 t \psitwo^4}{c \epsilon^2} \log \frac{12 e p}{t \epsilon}
  $, then with probability at least
  $
    1 - 2 \exp \left(
      -\frac{c}{2} \frac{\epsilon^2}{\psitwo^4} \min_{J \in \suppx} \frac{\norm{M_{\cdot J}}_F^2}{\norm{M_{\cdot J}}^2}
    \right)
  $ the matrix $MXW$ satisfies the RIP
  \[
    \begin{aligned}
      (1-\epsilon) \norm{z} & \le \norm{MXWz} \le (1+\epsilon) \norm{z} &
      & \text{for all }z\text{ with }\|z\|_0 \le t.
    \end{aligned}
  \]
\end{lemma}

\begin{proof}

Fix a support $T \subset [\palt]$ with $|T| = t$ and let $\Sigma_T \subset \real^{\palt}$ be the subspace of all vectors supported on $T$. By standard volumetric estimates \cite{BaraniukDavenportDeVoreEtAl2008,Vershynin2018} there is a $\frac{\epsilon}{4}$ cover $\mathcal{N}$ of the unit sphere in $\Sigma_T$ of cardinality
\[
  |\mathcal{N}|
  \le \left( \frac{12}{\epsilon} \right)^t.
\]
Since $\norm{Wz_i} = \norm{z_i}$, $z_i \in \mathcal{N}$, by Lemma \ref{lemma:X-concentration} and a union bound, we obtain
\[
  \pr{\exists z_i \in \mathcal{N} : \, \left| \norm{MXWz_i} - 1 \right| \ge \epsilon}
  \le 2 \left( \frac{12}{\epsilon} \right)^t \exp \left(-c \frac{\epsilon^2}{\psitwo^4} \min_{J \in \suppx} \frac{\norm{M_{\cdot J}}_F^2}{\norm{M_{\cdot J}}^2} \right).
\]
Let us assume that the event fails and thus $\left| \norm{MXWz_i} - 1 \right| \le \tau$ for all $z_i \in \mathcal{N}$. Then, by Lemma \ref{lemma:approximate-cover}, we have
\[
  \begin{aligned}
    (1-\epsilon) \norm{z} & \le \norm{MXWz} \le (1+\epsilon) \norm{z} &
    & \text{for all } z \in \Sigma_T.
  \end{aligned}
\]
There are $\binom{p}{t} \le \left( \frac{ep}{t} \right)^t$ supports $T$ of size $t$ and thus, by a union bound we obtain
\[
  \begin{aligned}
    (1-\epsilon) \norm{z} & \le \norm{MXWz} \le (1+\epsilon) \norm{z} &
    & \text{for all }z\text{ with }\|z\|_0 \le t
  \end{aligned}
\]
with probability of failure bounded by
\begin{multline*}
  2
  \left( \frac{ep}{t} \right)^t
  \left( \frac{12}{\epsilon} \right)^t
  \exp \left(
    -c \frac{\epsilon^2}{\psitwo^4} \min_{J \in \suppx} \frac{\norm{M_{\cdot J}}_F^2}{\norm{M_{\cdot J}}^2}
  \right)
  \\ =
  2
  \exp \left(
  -c \frac{\epsilon^2}{\psitwo^4} \min_{J \in \suppx} \frac{\norm{M_{\cdot J}}_F^2}{\norm{M_{\cdot J}}^2}
  + t \log \frac{12 e p}{t \epsilon}
  \right)
  \\ \le
  2
  \exp \left(
  -\frac{c}{2} \frac{\epsilon^2}{\psitwo^4} \min_{J \in \suppx} \frac{\norm{M_{\cdot J}}_F^2}{\norm{M_{\cdot J}}^2}
  \right)
\end{multline*}
if
\[
  t \log \frac{12 e p}{t \epsilon}
  \le
  \frac{c}{2} \frac{\epsilon^2}{\psitwo^4} \min_{J \in \suppx} \frac{\norm{M_{\cdot J}}_F^2}{\norm{M_{\cdot J}}^2}
  \Leftrightarrow
  \min_{J \in \suppx} \frac{\norm{M_{\cdot J}}_F^2}{\norm{M_{\cdot J}}^2}
  \ge
  \frac{2 t \psitwo^4}{c \epsilon^2} \log \frac{12 e p}{t \epsilon}.
\]

\end{proof}

\subsubsection{Null Space Property of \texorpdfstring{$AX$}{AX}}

The matrix $MS$ in the RIP results must have orthonormal columns, which is not generally true for $M=A$. However, this is true with a suitable preconditioner that we construct next. The null space property is invariant under preconditioning, which allows us to eliminate it, later.

\begin{lemma} \label{lemma:orthogonalize-columns}
  Let $M \in \real^{m \times q}$ with $m \ge q$ have full column rank. Then there is a matrix $T \in \real^{m \times m}$ with condition number $\kappa(T) = \kappa(M)$ such that $TM$ has orthonormal columns.
\end{lemma}

\begin{proof}

Let $M = U \Sigma V^T$ be the singular value decomposition of $M$. Define
\begin{align*}
  T & := D U^T, &
  D^{-1} := \operatorname{diag}[\sigma_1, \dots, \sigma_q, \sigma, \dots, \sigma]
\end{align*}
for $q \le m$ singular values $\sigma_i$ and remaining $m-q$ values $\sigma$ in the interval $[\sigma_1, \dots, \sigma_q]$. Then, we have
\begin{equation*}
  M^T T^T T M
  = (V \Sigma^T U^T) (U D^T) (D U^T) (U \Sigma V^T)
  = V \Sigma^T D^T D \Sigma V^T
  = V  V^T
  = I,
\end{equation*}
where we have used that $\Sigma^T D^T D \Sigma = I$. By construction, $T$ has singular values $\sigma_1, \dots, \sigma_q$ and one extra value $\sigma$ bounded by the former so that
\[
  \kappa(T) = \frac{\sigma_1}{\sigma_q} = \kappa(M).
\]

\end{proof}

\begin{lemma} \label{lemma:preconditioned-stable-rank}
  Let $A \in \real^{m \times n}$ and $T \in \real^{m \times m}$ be invertible. Then
  \[
    \frac{\norm{A}_F}{\norm{A}}
    \le \kappa(T) \frac{\norm{TA}_F}{\norm{TA}}.
  \]
\end{lemma}

\begin{proof}

We first show that
\[
  \norm{TA}_F \ge \norm{T^{-1}}^{-1} \norm{A}_F.
\]
Indeed $\norm{x} \le \norm{T^{-1}} \norm{Tx}$ implies $\norm{Tx} \ge \norm{T^{-1}}^{-1} \norm{x}$ and thus applied to the columns $a_j$ of $A$, we have
\[
  \norm{TA}_F^2
  = \sum_{j=1}^n \norm{T a_j}^2
  \ge \sum_{j=1}^n \norm{T^{-1}}^{-2} \norm{a_j}^2
  = \norm{T^{-1}}^{-2} \norm{A}_F^2.
\]
With this estimate, we obtain
\[
  \kappa(T) \frac{\norm{TA}_F}{\norm{TA}}
  \ge \norm{T} \norm{T^{-1}} \frac{\norm{T^{-1}}^{-1} \norm{A}_F}{\norm{T} \norm{A}}
  = \frac{\norm{A}_F}{\norm{A}}.
\]

\end{proof}

\begin{corollary} \label{cor:X-rip}
  Let $W \in \real^{p \times \palt}$ be an isometry and for $X$, $S$ and $D$ defined in \eqref{eq:X-def}, assume that $AS$ has full column rank and
  $
    \min_{J \in \suppx} \frac{\norm{A_{\cdot J}}_F^2}{\norm{A}_{\cdot J}^2}
    \ge \frac{2 t \psitwo^4}{c \epsilon^2} \kappa(AS) \log \frac{12 e p}{t \epsilon}
  $
  . Then there is an invertible matrix $T \in \real^{m \times m}$ so that with the diagonal scaling $D_j = \norm{TA_{\cdot J}}_F^{-1}$ for all $j$ in block $J \in \suppx$ with probability at least
  $
    1 - 2 \exp \left(
      -\frac{c}{2} \frac{\epsilon^2}{\psitwo^4} \frac{1}{\kappa(AS)} \min_{J \in \suppx} \frac{\norm{A_{\cdot J}}_F^2}{\norm{A_{\cdot J}}^2}
    \right)
    $ the matrix $TAXW$ satisfies the RIP
  \[
    \begin{aligned}
      (1-\epsilon) \norm{z} & \le \norm{TAXWz} \le (1+\epsilon) \norm{z} &
      & \text{for all }z\text{ with }\|z\|_0 \le t.
    \end{aligned}
  \]
\end{corollary}

\begin{proof}

Since the matrix $AS$ has full column rank by Lemmas \ref{lemma:orthogonalize-columns} and \ref{lemma:preconditioned-stable-rank}, there is an invertible matrix $T$ such that
\begin{align*}
  \kappa(T) & = \kappa(AS), &
  & TAS\text{ has orthogonal columns} \\
  \frac{\norm{A_{\cdot J}}_F}{\norm{A_{\cdot J}}} & \le \kappa(T) \frac{\norm{TA_{\cdot J}}_F}{\norm{TA_{\cdot J}}} &
  & \text{for all } J \in \suppx.
\end{align*}
Thus, the corollary follows from Lemma \ref{lemma:X-rip} with $M = TA$.

\end{proof}

The last corollary allows us to recover $x = S\one$ by $\ell_1$-minimization
\[
  \begin{aligned}
    & \min_{x \in \real^n} \|x\|_1 & & \text{subject to} & TA x & = b,
  \end{aligned}
\]
preconditioned by some matrix $T$. This problem is not yet solvable by the student, who generally has no access to the matrix $T$, which is only used by the teacher for the construction of $X$. However, the matrix $T$ is unnecessary for $\ell_1$ recovery because the RIP implies the null space property, which is sufficient for recovery and independent of left preconditioning.

\begin{corollary} \label{cor:X-nsp}
  Let $W \in \real^{p \times \palt}$ be an isometry and for $X$, $S$ and $D$ defined in \eqref{eq:X-def},  assume that $AS$ has full column rank and
  $
    \min_{J \in \suppx} \frac{\norm{A_{\cdot J}}_F^2}{\norm{A_{\cdot J}}^2}
    \ge \frac{2 t \psitwo^4}{c \epsilon^2} \kappa(AS) \log \frac{12 e p}{t \epsilon}
  $
  . Then there is an invertible matrix $T \in \real^{m \times m}$ so that with the diagonal scaling $D_j = \norm{TA_{\cdot J}}_F^{-1}$ for all $j$ in block $J \in \suppx$ with probability at least
  $
    1 - 2 \exp \left(
      -\frac{c}{2} \frac{\epsilon^2}{\psitwo^4} \frac{1}{\kappa(AS)} \min_{J \in \suppx} \frac{\norm{A_{\cdot J}}_F^2}{\norm{A_{\cdot J}}^2}
    \right)
  $ the matrix $AXW$ satisfies the null space property of order $t$
  \begin{align*}
    \norm{z_T}_1 & < \norm{z_{\bar{T}}}_1 &
    & \text{for all }z \in \operatorname{ker}(AXW)\text{ and }T \subset [p], \, |T| \le t.
  \end{align*}
  with complement $\bar{T}$ of $T$.
\end{corollary}

\begin{proof}

Setting $\epsilon = \frac{1}{3}$, changing $t \to 2t$ and adjusting the constants accordingly, with the given conditions and probabilities, the matrix $TAX$ satisfies the $\left(2t, \frac{1}{3} \right)$-RIP. Thus, by \cite{FoucartRauhut2013}, proof of Theorem $6.9$, $TAX$ satisfies
\begin{align*}
  \norm{z_T}_1 & < \frac{1}{2} \norm{z}_1 &
  & \text{for all }z \in \operatorname{ker}(TAX)\text{ and }T \subset [p], \, |T| \le t.
\end{align*}
This directly implies the null space property of order $t$
\begin{align*}
  \norm{z_T}_1 & < \norm{z_{\bar{T}}}_1 &
  & \text{for all }z \in \operatorname{ker}(TAX)\text{ and }T \subset [p], \, |T| \le t.
\end{align*}
Since $T$ is invertible, $\operatorname{ker}(TAX) = \operatorname{ker}(AX)$, so that also $AX$ satisfies the null space property.

\end{proof}

\begin{remark}
  For Corollaries \ref{cor:X-rip} and \ref{cor:X-nsp}, we are particularly interested in applications where $x = S\one$ is the global $\ell_0$-minimizer of $Ax = b$ in \ref{eq:l0-min-classes}. Then the full column rank condition of $AS$ is automatically satisfied by Lemma \ref{appendix:lemma:l0-min-split-independent}.
\end{remark}

\subsection{Model Tree: Proposition \ref{prop:class-tree}}
\label{appendix:class-tree}

\begin{proposition}[Proposition \ref{prop:class-tree} restated]
  \PropClassTree{\nolabel}
\end{proposition}

\begin{proof}

We build a matrix $X$ according to \ref{item:model-first} - \ref{item:model-last} and use the extra matrix $W$ in Corollary \ref{cor:X-nsp} to build a tree out of it. By assumption, the support of $x$ is partitioned into patches $\{J_1, \dots, J_{q}\} = \suppx$ for which we define the corresponding partition $\suppz = \{K_1, \dots, K_q\}$ of $[p]$ and $Z$ by
\[
  Z_{kl} := \left\{ \begin{array}{ll}
    1   & k = k_l \\
    0   & \text{else}
  \end{array} \right.
\]
for some choices $k_l \in K_l$. The index sets $\suppx$ and $\suppz$ are naturally combined by their indices to obtain the pairs $\suppxz$. With these choices, the matrix $X$ is given by \ref{item:model-first} - \ref{item:model-last}.

$X$ is non-zero only on blocks $[J,K] \in \suppxz$, which allows us to build a tree, whose nodes we index by $i$ in a suitable index set $\indexset$. Each node $i$ is associated with a subset $K_i \subset [q]$ that is a union of two children $K_i = \bigcup_{j \in \child(i)} K_j$, starting with leave nodes $K_i \in \suppz$, e.g.
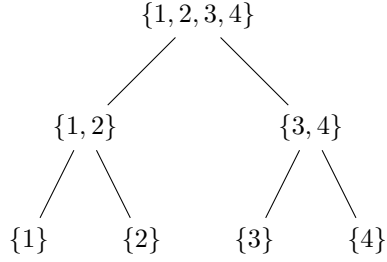
\begin{center}
  \begin{tikzpicture}
    \node {$\{1,2,3,4\}$} [sibling distance=3cm]
    child{
      node {$\{1,2\}$} [sibling distance=1.5cm]
      child{node {$\{1\}$}}
      child{node {$\{2\}$}}
    }
    child {
      node {$\{3,4\}$} [sibling distance=1.5cm]
      child{node {$\{3\}$}}
      child{node {$\{4\}$}}
    };
  \end{tikzpicture}
\end{center}
We now define matrices $X_i$ on each node, starting with the leaves
\[
  X_i := X_{\cdot K_i}
\]
for leave $i$ and then inductively by joining the two child matrices
\begin{align*}
  X_i & := \begin{bmatrix} X_{j_1} & X_{j_2} \end{bmatrix} \bar{W}_i, &
  \bar{W}_i & = \frac{1}{\sqrt{2}} \begin{bmatrix} I_{K_{j_1}, K_{j_1}} \\ I_{K_{j_2}, K_{j_2}}\end{bmatrix}
\end{align*}
for $\child(i) = \{j_1, j_2\}$. It is easy to join all $\bar{W_i}$ matrices leading up to node $i$ into a single isometry $W_i$ so that
\[
  X_i = \begin{bmatrix} X_1 & \cdots & X_q \end{bmatrix} W_i.
\]
which implies
\begin{align*}
  X_{\child(i)}
  & = \begin{bmatrix} X_1 & \cdots & X_q \end{bmatrix} W_{\child(i)},
  &
  W_{\child(i)}
  & = \begin{bmatrix} W_{j_1} & W_{j_2} \end{bmatrix},
\end{align*}
where again $W_{\child(i)}$ is an isometry because the columns of $W_{j_1}$ and $W_{j_2}$ have non-overlapping support. By Lemma \ref{lemma:n-nodes} the tree has at most $2^{L+1}$ nodes and thus, if
\begin{equation} \label{eq:rip-condition-tree}
    \min_{J \in \suppx} \frac{\norm{A_{\cdot J}}_F^2}{\norm{A_{\cdot J}}^2}
    \ge \frac{2 t \psitwo^4}{c \epsilon^2} \kappa(AS) \log \frac{12 e p}{t \epsilon}
\end{equation}
by Corollary \ref{cor:X-nsp} and union bound over all tree nodes, with probability at least
\[
    1 - 42^L \exp \left(
      -\frac{c}{2} \frac{\epsilon^2}{\psitwo^4} \frac{1}{\kappa(AS)} \min_{J \in \suppx} \frac{\norm{A_{\cdot J}}_F^2}{\norm{A_{\cdot J}}^2}
    \right)
\]
all nodes $X_{\child(i)}$ satisfy the $t$-NSP. For this probability to be close to one, $\log 2^L$ must be smaller than say half the exponent
\begin{align*}
  L \gtrsim \log 2^L
  & \le -\frac{c}{4} \frac{\epsilon^2}{\psitwo^4} \frac{1}{\kappa(AS)} \min_{J \in \suppx} \frac{\norm{A_{\cdot J}}_F^2}{\norm{A_{\cdot J}}^2} &
  & \Leftrightarrow &
  \min_{J \in \suppx} \frac{\norm{A_{\cdot J}}_F^2}{\norm{A_{\cdot J}}^2}
  \gtrsim \frac{t \psitwo^4}{\epsilon^2} \kappa(AS) \log s.
\end{align*}
Combining this with the NSP condition \eqref{eq:rip-condition-tree}, if
\[
    \min_{J \in \suppx} \frac{\norm{A_{\cdot J}}_F^2}{\norm{A_{\cdot J}}^2}
    \gtrsim \frac{t \psitwo^4}{\epsilon^2} \kappa(AS) L
    + \frac{t \psitwo^4}{\epsilon^2} \kappa(AS) \log \frac{12 e p}{t \epsilon},
\]
with probability at least
\[
    1 - 2 \exp \left(
      -\frac{c}{2} \frac{\epsilon^2}{\psitwo^4} \frac{1}{\kappa(AS)} \min_{J \in \suppx} \frac{\norm{A_{\cdot J}}_F^2}{\norm{A_{\cdot J}}^2}
    \right)
\]
all nodes $X_{\child(i)}$ satisfy the $t$-NSP. This yields the statements in the proposition if we choose $\epsilon \sim 1$ and $\psitwo \sim 1$, without loss of generality.

Let us verify the remaining properties of learnable trees. By construction, we have $t/\ttrain = 2$ and $\gamma=2$. Since all random samples in $X$ are absolutely continuous with respect to the Lebesgue measure, the probability of rank deficit $X_i$ is zero. The remaining assumptions are given, with the exception of the first two inequalities in \ref{assumption:sparse-factorization-alternative}. Renaming the number of training samples $q$, whose name is already used otherwise here, to $r$, they state that $t \ge c \log r $ and $r > c p^2 \log^2 p$ and thus imply that $t \ge \log p^2 + \log^3 p$, which is sufficient since the number of training samples $r$ is at the disposal of the teacher.

\end{proof}

\section{Technical Supplements}

\begin{lemma}
  \label{lemma:random-matrix-norm}
  Let $R \in \real^{n \times p}$ be a i.i.d. random matrix with mean zero entries of variance one. Then for any  $A \in \real^{m \times n}$ and $u \in \real^p$ we have
  \[
    \E{\|ARu\|^2} = \|A\|_F^2 \|u\|^2.
  \]
\end{lemma}

\begin{proof}

Since $\E{R_{ik}R_{jl}} = \delta_{ij} \delta_{kl}$, we have
\begin{align*}
  \E{\|ARu\|^2}
  & = \E{\dualp{ARu, ARu}}
  \\
  & = \E{\sum_{ijkl} u_k R_{ik}(A^TA)_{ij} R_{jl} u_l}
  \\
  & = \sum_{ijkl} (A^TA)_{ij}  u_k u_l \E{R_{ik}R_{jl}}
  \\
  & = \sum_{ik} (A^TA)_{ii}  u_k u_k
  \\
  & = \|A\|_F^2 \|u\|^2.
\end{align*}

\end{proof}

\begin{lemma} \label{lemma:Ax-b-expectation}
  Let $A \in \real^{m \times n}$ be a matrix, $b \in \real^m$ be a vector and $x \in \real^n$ a i.i.d. random vector with $\E{x_j} = 0$, $\E{x_j^2} = 1$. Then
  \[
    \E{\norm{Ax + b}^2}
    =  \norm{A}_F^2 + \norm{b}^2.
  \]
\end{lemma}

\begin{proof}

Since $b$ is not random, we have
\[
  \E{\norm{Ax + b}^2}
  = \E{\norm{Ax}^2} + \norm{b}^2
  =  \norm{A}_F^2 + \norm{b}^2,
\]
where in the last equality we have used Lemma \ref{lemma:random-matrix-norm} with $\real^{n \times 1}$ matrix $R=x$ and $u = [1] \in \real^1$.

\end{proof}

The following result is a slight variation of \cite[Theorem $6.3.2$]{Vershynin2018}.
\begin{lemma} \label{lemma:Ax-b-concentration}
  Let $A \in \real^{m \times n}$ be a matrix, $b \in \real^m$ be a vector and $x \in \real^n$ a i.i.d. random vector with $\E{x_j} = 0$, $\E{x_j^2} = 1$ and $\norm{x}_{\psi_2} \le \psitwo$. Then
  \begin{multline*}
    \pr{ \left|\norm{Ax + b}^2 - \norm{A}_F^2 - \norm{b}^2 \right| \ge \epsilon \left( \norm{A}_F^2 + \norm{b}^2 \right) }
    \\
    \le 8 \exp \left[ -c \min(\epsilon^2, \epsilon) \frac{\|A\|_F^2 + \|b\|^2}{\psitwo^4 \|A\|^2} \right].
  \end{multline*}
\end{lemma}

\begin{proof}

We decompose
\begin{align*}
  \norm{Ax + b}^2 - \norm{A}_F^2 - \norm{b}^2
  & = \norm{Ax}^2 + 2\dualp{Ax, b} + \norm{b}^2 - \norm{A}_F^2 - \norm{b}^2
  \\
  & = \left(\norm{Ax}^2 - \norm{A}_F^2\right) + 2\dualp{Ax, b}
\end{align*}
so that
\begin{multline*}
  \pr{ \pm \left(\norm{Ax + b}^2 - \norm{A}_F^2 - \norm{b}^2 \right) \ge \epsilon \left( \norm{A}_F^2 + \norm{b}^2 \right) }
  \\
  \begin{aligned}
    & \le \pr{ \pm \left(\norm{Ax}^2 - \norm{A}_F^2 \right) \pm 2 \dualp{Ax,b} \ge \epsilon \left( \norm{A}_F^2 + \norm{b}^2 \right) }
    \\
    & \le \pr{ \pm \left(\norm{Ax}^2 - \norm{A}_F^2 \right) \ge \epsilon \norm{A}_F^2 }
    + \pr{ \pm 2 \dualp{Ax, b} \ge \epsilon \norm{b}^2 }.
  \end{aligned}
\end{multline*}
It remains to estimate the two probabilities on the right hand side. Since $\E{x_j^2} = 1$, we have $\psitwo \gtrsim 1$ and thus from the proof of Theorem $6.3.2$ in \cite{Vershynin2018}, we have
\[
  \pr{\pm \left( \|A x\|^2 - \|A\|_F^2 \right) \ge \epsilon \|A\|_F^2}
  \le 2 \exp \left[ -c \min(\epsilon^2, \epsilon) \frac{\|A\|_F^2}{\psitwo^4 \|A\|^2} \right]
\]
and from Hoeffding's inequality, we have
\[
  \pr{\pm 2 \dualp{Ax, b} \ge \epsilon \|b\|^2 }
  \le 2 \exp \left[ -c \epsilon^2 \frac{\|b\|^4}{\psitwo^2 \|A^T b\|^2} \right]
  \le 2 \exp \left[ -c \epsilon^2 \frac{\|b\|^2}{\psitwo^4 \|A^T\|^2} \right].
\]

\end{proof}

The following result is a slight variation of \cite[Theorem $6.3.2$]{Vershynin2018}.
\begin{lemma} \label{lemma:Ax-b-psi2}
  Let $A \in \real^{m \times n}$ be a matrix, $b \in \real^m$ be a vector and $x \in \real^n$ a i.i.d. random vector with $\E{x_j} = 0$, $\E{x_j^2} = 1$ and $\norm{x}_{\psi_2} \le \psitwo$. Then
  \[
    \norm{ \norm{Ax + b} - \left(\norm{A}_F^2 + \norm{b}^2 \right)^{1/2} }_{\psi_2} \le C \psitwo^2 \norm{A}
  \]
  for some constant $C \ge 0$.
\end{lemma}

\begin{proof}

We use a standard argument, e.g. from the proof of Theorem $6.3.2$ in \cite{Vershynin2018}. An elementary computation shows that for $\delta^2 = \min(\epsilon^2, \epsilon)$ and any $a,b \in \real$, we have
\[
  \begin{aligned}
    |a - b| & \ge \delta b, &
    & \Rightarrow &
    |a^2 - b^2| \ge \epsilon b^2.
  \end{aligned}
\]
With $a = \norm{Ax + b}$ and $b = \left( \norm{A}_F^2 + \norm{b}^2 \right)^{1/2}$ and Lemma \ref{lemma:Ax-b-concentration}, this implies
\begin{multline*}
  \pr{ \left|\norm{Ax + b} - \left(\norm{A}_F^2 - \norm{b}^2 \right)^{1/2} \right| \ge \delta \left( \norm{A}_F^2 + \norm{b}^2 \right)^{1/2} }
  \\
  \le 8 \exp \left[ -c \delta^2 \frac{\|A\|_F^2 + \|b\|^2}{\psitwo^4 \|A\|^2} \right].
\end{multline*}
This shows Subgaussian concentration and thus the $\psi_2$-norm of the lemma.

\end{proof}

\end{document}